\pdfoutput=1 % MUST stay within the first 5 lines: tells arXiv to use pdfLaTeX.
\documentclass[11pt]{article}

\usepackage[font=libertinus, citestyle=numeric]{kurbanlab}

\DeclareAffiliation{hbku}{%
  College of Science and Engineering, Hamad Bin Khalifa University, Doha, Qatar}

\DeclareAffiliation{tamu}{%
  Department of Computer and Electrical Engineering,
  Texas A\&M University, College Station, TX, USA}

\DeclareAffiliation{iub}{%
  Luddy School of Informatics, Computing, and Engineering,
  Indiana University Bloomington, Bloomington, IN, USA}

\usepackage{pgfplots}
\pgfplotsset{compat=1.11}
\usetikzlibrary{arrows.meta,positioning,fit,backgrounds,calc}

\newcommand{\rea}{\mathrm{REA}}
\newcommand{\rmg}{\mathrm{RM}}

\title{When Does Consensus Mean Correctness? Measuring the Agreement--Accuracy Coupling with Semantics-Preserving Re-Rendering}
\RunningTitle{When does consensus mean correctness?}

\Author{Rasul Khanbayov}{hbku}
\Author[corresponding=hkurban@hbku.edu.qa]{Hasan Kurban}{hbku}
\CodeURL{https://anonymous.4open.science/r/rendeq-0170}
\Venue{Under review}

\begin{document}
\maketitle

\begin{abstract}
A model's agreement across perturbed inputs is used both as a label-free reliability signal and as a self-training target, on the premise that agreement tracks correctness. That coupling is rarely measured directly: natural-image perturbations preserve meaning only by assumption, and no exact answer key localizes errors. Scientific figures remove both obstacles, a figure is drawn from data by a program, so redrawing it yields images that are semantically equivalent by construction and share a programmatically exact answer. We build \textbf{RENDEQ}, a generator of such render-equivalence sets, and measure the coupling on three open-weight VLMs, checking every finding across three independent instantiations. Re-rendering beats resampling on both accuracy and reliability. Agreement beats an evidence-carrying baseline, mean token log-probability, on two of three models and ties on the third, reversing an intermediate, buggy replication traced to a rendering-pipeline failure. The dispersion behind this is concentrated in one style factor, the plotting library, more than double the next-largest factor and an order of magnitude above the noise floor. Fine-tuning on the model's own cross-render consensus inverts: accuracy falls in every one of five replication runs, the opposite sign to published results on natural images. Agreement certifies correctness only above a threshold set by how diffuse a model's errors are, and an objective that rewards agreement destroys exactly that diffuseness.
\end{abstract}

\printkeywords

%%%%%%%%%%%%%%%%%%%%%%%%%%%%%%%%%%%%%%%%%%%%%%%%%%%%%%%%%%%%%%%%%%%%%%%%%%%%%%%%
\section{Introduction}
\label{sec:intro}

Two lines of work have converged on the same idea. The first treats a model's agreement across perturbed inputs as a label-free reliability signal: \citet{zhang2024vluncertainty} blur images and rephrase questions and read hallucination off the entropy of the resulting answer clusters, and \citet{rosenfeld2025stability} show at scale that a sample's stability under benign perturbations predicts whether it was answered correctly. The second turns the same consensus into a training target: \citet{kaya2026efficient} aggregate over augmented inputs at test time and then fine-tune on the resulting consensus pseudo-labels, reporting gains across nine benchmarks. Both uses rest on one premise, that agreement tracks correctness.

That premise is conditional, and the condition has been stated before. Agreement estimates how concentrated a model's answer distribution is, aggregation estimates where its mode sits, and only the second can be wrong about the answer. When errors are correlated rather than idiosyncratic, majority voting locks in the wrong answer instead of correcting it \citep{estornell2024debate}, and consensus certifies a property of the model-induced distribution rather than semantic correctness. So the useful question is not whether consensus works, but how tightly concentration and correctness are coupled in a given model on given data.

That coupling is almost never measured, for two reasons: perturbations of natural images preserve meaning only by assumption, so a change in the answer cannot be attributed to model instability rather than to the perturbation having altered the content, and there is no exact answer key for the intermediate quantities a model must read, so the source of an error cannot be localized. Scientific figures remove both obstacles. A figure is the output of a rendering program applied to data; the data fixes the correct answer to any well-posed quantitative question, and the rendering style is a nuisance variable a faithful reader should ignore. Redrawing the same data therefore yields images that are semantically equivalent \emph{by construction} and that share an exact, programmatically known answer, together with exact ground truth for every intermediate read. We call such a family a \emph{render-equivalence set} (\cref{fig:teaser}). It serves as a measurement instrument: correctness becomes exactly observable and concentration estimable without confounding, so the coupling between them can be studied rather than assumed.

We build that instrument (RENDEQ), define the two quantities it supports, \emph{render-equivalence agreement} ($\rea$) and \emph{render marginalization} ($\rmg$), and use them to estimate the coupling on Qwen2.5-VL-7B, Qwen2.5-VL-3B, and InternVL2-8B. We report four results, checked across three independently generated RENDEQ instantiations: re-rendering beats resampling on both accuracy and reliability; agreement beats an evidence-carrying baseline on two of three models and ties on the third; the dispersion behind all of this is concentrated in one style factor rather than spread across many; and fine-tuning on cross-render consensus inverts, cutting accuracy rather than raising it. \Cref{sec:theory} derives each as a prediction, \cref{tab:predictions} sets them against the outcomes, and \cref{sec:experiments} reports the numbers and confidence intervals in full.

\paragraph{Contributions} Our contributions are the instrument, the measurement, and the boundary condition:
\begin{enumerate}
\item \textbf{RENDEQ}, a generator and protocol whose exact semantic equivalence, exact programmatic ground truth, and individually togglable style factors make correctness exactly observable and concentration estimable without confounding (\cref{sec:benchmark}).
\item A measurement of the coupling, with per-factor attribution identifying which perturbation carries it (\cref{sec:experiments}); correctness is observed exactly, concentration estimated subject to the bias of \cref{prop:concentration}(ii).
\item A boundary condition on consensus self-training: an empirical inversion of a published positive result, with the mechanism that explains it (\cref{sec:selftraining}).
\item An analysis (\cref{sec:theory}) with three results: agreement is upward-biased for concentration at finite $K$, with an unbiased pairwise alternative; it certifies correctness only above a threshold set by how diffuse the model's errors are; and render dispersion decomposes over style factors, bounding what any single-factor ablation can remove.
\end{enumerate}

We claim neither priority over input-perturbation consistency \citep{zhang2024vluncertainty,kaya2026efficient} nor novelty for the concentration-versus-correctness distinction itself, which the self-consistency literature states. What re-rendering adds is exactness, and what exactness buys is measurement.

%%%%%%%%%%%%%%%%%%%%%%%%%%%%%%%%%%%%%%%%%%%%%%%%%%%%%%%%%%%%%%%%%%%%%%%%%%%%%%%%

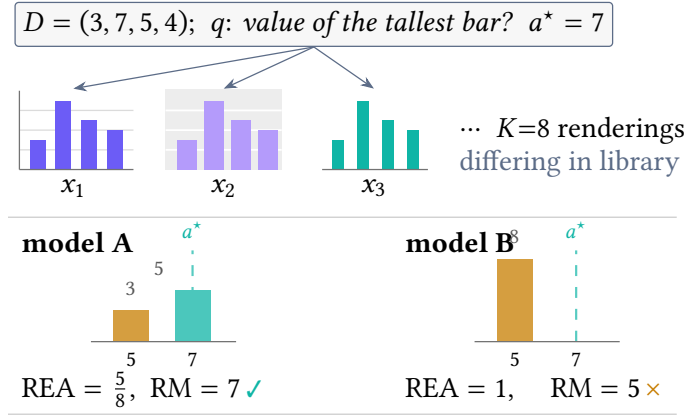
\begin{figure}[t]
\centering
\resizebox{0.6\columnwidth}{!}{%
\begin{tikzpicture}[
  font=\scriptsize,
  >={Stealth[length=1.3mm]},
  line width=0.35pt
]

% ---------- the data and its exact answer ----------
\node[
  draw=kilslate,
  rounded corners=1pt,
  fill=black!3,
  inner sep=2.2pt
] (data) at (0.35,0)
{$D=(3,7,5,4)$;\ \ $q$: \emph{value of the tallest bar?}\ \ $a^\star=7$};

% ---------- three of the K renderings ----------
\begin{scope}[shift={(-2.55,-1.40)}]
  \draw[black!14] (0.06,0.19) -- (1.16,0.19);
  \draw[black!14] (0.06,0.38) -- (1.16,0.38);
  \draw[black!14] (0.06,0.57) -- (1.16,0.57);
  \foreach \i/\v in {0/3,1/7,2/5,3/4}
    \fill[kilviolet]
      ({0.16+\i*0.25},0)
      rectangle
      ({0.16+\i*0.25+0.15},{\v*0.095});
  \draw[black!55] (0.06,0) -- (1.16,0);
  \draw[black!55] (0.06,0) -- (0.06,0.76);
\end{scope}

\begin{scope}[shift={(-1.10,-1.40)}]
  \fill[black!7] (0.02,0) rectangle (1.14,0.78);
  \draw[white, line width=0.5pt] (0.02,0.19) -- (1.14,0.19);
  \draw[white, line width=0.5pt] (0.02,0.38) -- (1.14,0.38);
  \draw[white, line width=0.5pt] (0.02,0.57) -- (1.14,0.57);
  \foreach \i/\v in {0/3,1/7,2/5,3/4}
    \fill[killilac]
      ({0.13+\i*0.26},0)
      rectangle
      ({0.13+\i*0.26+0.19},{\v*0.095});
\end{scope}

\begin{scope}[shift={(0.35,-1.40)}]
  \foreach \i/\v in {0/3,1/7,2/5,3/4}
    \fill[kilteal]
      ({0.18+\i*0.24},0)
      rectangle
      ({0.18+\i*0.24+0.11},{\v*0.095});
  \draw[black!55] (0.08,0) -- (1.10,0);
\end{scope}

\node at (-1.97,-1.58) {$x_1$};
\node at (-0.52,-1.58) {$x_2$};
\node at (0.93,-1.58) {$x_3$};

\node[anchor=west] at (1.62,-1.02)
  {$\cdots$\ \ $K{=}8$ renderings};

\node[anchor=west, kilslate] at (1.62,-1.34)
  {differing in library};

\draw[->,kilslate] (data.south) -- (-1.97,-0.62);
\draw[->,kilslate] (data.south) -- (-0.52,-0.62);
\draw[->,kilslate] (data.south) -- (0.93,-0.62);

\draw[black!20] (-2.60,-1.86) -- (3.85,-1.86);

% ---------- answer distributions ----------
\node[anchor=west] at (-2.60,-2.10) {\textbf{model A}};

\begin{scope}[shift={(-2.10,-3.06)}]
  \draw[kilteal!70, dashed, line width=0.6pt]
    (1.28,0) -- (1.28,0.88);
  \node[kilteal, font=\tiny, anchor=south]
    at (1.28,0.86) {$a^\star$};

  \fill[kilamber!75] (0.52,0) rectangle (0.86,0.30);
  \fill[kilteal!75] (1.11,0) rectangle (1.45,0.50);

  \node[font=\tiny, black!65, anchor=south west]
    at (0.50,0.31) {3};
  \node[font=\tiny, black!65, anchor=south east]
    at (1.14,0.51) {5};

  \draw[black!55] (0.30,0) -- (1.70,0);

  \node[font=\tiny] at (0.69,-0.17) {5};
  \node[font=\tiny] at (1.28,-0.17) {7};
\end{scope}

\node[anchor=west] at (-2.60,-3.52)
{$\rea=\tfrac58$,\ \ $\rmg=7$\,\textcolor{kilteal}{\checkmark}};

\node[anchor=west] at (1.10,-2.10) {\textbf{model B}};

\begin{scope}[shift={(1.60,-3.06)}]
  \draw[kilteal!70, dashed, line width=0.6pt]
    (1.28,0) -- (1.28,0.88);
  \node[kilteal, font=\tiny, anchor=south]
    at (1.28,0.86) {$a^\star$};

  \fill[kilamber!75] (0.52,0) rectangle (0.86,0.80);

  \node[font=\tiny, black!65, anchor=south west]
    at (0.50,0.81) {8};

  \draw[black!55] (0.30,0) -- (1.70,0);

  \node[font=\tiny] at (0.69,-0.17) {5};
  \node[font=\tiny] at (1.28,-0.17) {7};
\end{scope}

\node[anchor=west] at (1.10,-3.52)
{$\rea=1$,\ \ \ \ \ $\rmg=5$\,\textcolor{kilamber}{$\times$}};

\draw[black!20] (-2.60,-3.76) -- (3.85,-3.76);

\end{tikzpicture}%
}
\caption{\textbf{Agreement measures the height of the tallest bar; correctness asks where it stands.}
One dataset is drawn under $K$ styles that leave the answer unchanged, so any disagreement is the model's. Histograms are answer distributions over the $K{=}8$ renderings: $\rea$ reads the tallest bar's share, $\rmg$ its position, and only the second is comparable with $a^\star$. Model A is split yet right; model B is unanimous and wrong, so agreement ranks B above A, which is why \cref{prop:lie} conditions on how diffuse a model's errors are. Renderings differ in plotting library (\cref{fig:results}b); values are illustrative.}
\label{fig:teaser}
\end{figure}

%%%%%%%%%%%%%%%%%%%%%%%%%%%%%%%%%%%%%%%%%%%%%%%%%%%%%%%%%%%%%%%%%%%%%%%%%%%%%%%%
\section{Related work}
\label{sec:related}

\paragraph{Consensus as a reliability signal.} Self-consistency over reasoning paths \citep{wang2023selfconsistency}, resampled responses and task decompositions \citep{khan2024consistency,yang2024decc}, semantic clusters \citep{kuhn2023semantic,farquhar2024semantic}, and geometric quantities inside a pipeline \citep{kim2026zoomconsistency} all read reliability off agreement. Closest are the methods that perturb the \emph{input}: \citet{zhang2024vluncertainty} blur images and rephrase questions, and \citet{rosenfeld2025stability} show at scale that stability under benign perturbations predicts correctness. \citet{park2026vauq} instead score dependence on visual evidence, and \citet{xiao2026vlcalibration} calibrate with training and labels. Our object is not another such score but the premise they share; \cref{tab:novelty} in \cref{app:positioning} places us against them. That the premise can fail is itself understood: agreement certifies a property of the model's answer distribution rather than its correctness, and \citet{estornell2024debate} analyze how correlated errors let majority voting settle on a wrong answer. Our analysis restates this for input perturbation and adds the continuous case; our contribution is not the distinction but a setting in which one side is exactly observable and the other estimable without confounding, so the two can be related empirically.

\paragraph{Consensus as a training signal.} \citet{kaya2026efficient} aggregate over semantics-preserving input augmentations at test time and then fine-tune on the resulting consensus pseudo-labels, reporting gains across nine benchmarks; classical test-time augmentation aggregates over input perturbations in embedding space \citep{zanella2024mta}. \Cref{sec:selftraining} reports the opposite outcome for the fine-tuning half of this recipe and identifies the regime that separates the two results.

\paragraph{Chart and figure reasoning.} Benchmarks are mature \citep{kahou2017figureqa,methani2020plotqa,masry2022chartqa}, and recent work raises accuracy via programmatic synthesis and verifiable-reward training \citep{liu2026chartverse,zhang2026chartrl,sinha2025chartrvr,kondic2025chartgen}: these synthesize \emph{diverse} charts and need labels for the reward, whereas we hold data fixed and use none at inference. A parallel line documents degradation under perturbation and restyling \citep{mukhopadhyay2024robustcqa,shin2025losingplot,zhao2026robustcot}, establishing our premise; work separating perception from reasoning \citep{xiao2025perceptionr1} and probing evidence-versus-prior conflict \citep{wang2026vfat,chen2026cdhbench,singla2026seeorguess} needs curated conflict sets, whereas our decomposition is label-free on non-conflicting figures. Selective prediction and calibration provide the evaluation framing \citep{elyaniv2010foundations,guo2017calibration,srinivasan2024selective}, and test-time compute scaling motivates the cost $\rmg$ pays \citep{snell2024scaling}.

%%%%%%%%%%%%%%%%%%%%%%%%%%%%%%%%%%%%%%%%%%%%%%%%%%%%%%%%%%%%%%%%%%%%%%%%%%%%%%%%
\section{Render-equivalence sets}
\label{sec:method}

Let $D$ be the data underlying a figure and $q$ a quantitative question with an exact answer $a^\star(D,q)$ computable from $D$ by a deterministic program. A renderer $R$ maps data and a style configuration $\theta \in \Theta$ to an image $x = R(D;\theta)$, where $\theta$ controls nuisance factors that do not change the encoded data: plotting library, palette, theme, gridlines, aspect ratio, marker shape, font, ticks, legend, resolution.

\begin{definition}[Render-equivalence set]
\label{def:renderset}
Given $(D,q)$ and styles $\theta_1,\dots,\theta_K$, the render-equivalence set is $\mathcal{X}(D,q) = \{x_i = R(D;\theta_i)\}_{i=1}^K$. By construction every $x_i$ shares the answer $a^\star(D,q)$.
\end{definition}

A VLM $f$ produces a parsed answer $\hat y_i = \mathrm{parse}(f(x_i,q))$ per rendering; $f$ is a black box and only its emitted answers are used.

\paragraph{Agreement ($\rea$).} For categorical answers, with $\hat m$ the modal answer among $\{\hat y_i\}$,
\begin{equation}
\rea_{\text{cat}} = \tfrac{1}{K}\textstyle\sum_{i}\mathbf{1}[\hat y_i = \hat m] \in (0,1].
\end{equation}
For continuous answers, with $\tilde y$ the median, relative tolerance $\tau$, and small $\varepsilon>0$ guarding the division,
\begin{equation}
\rea_{\text{num}} = \tfrac{1}{K}\textstyle\sum_{i}\mathbf{1}\!\left[\tfrac{|\hat y_i - \tilde y|}{|\tilde y|+\varepsilon} \le \tau\right].
\end{equation}
Both are computed without ground truth. At $K$ renderings $\rea$ takes at most $K$ distinct values, bounding the resolution of any threshold rule built on it.

\paragraph{Aggregation ($\rmg$) and error decomposition.} We take the mode for categorical answers, ties broken by first-seen rendering order, and the median for continuous answers. $\rmg$ uses no labels and costs $K$ forward passes. We aggregate at the answer level; \citet{kaya2026efficient} report that token-level aggregation is stronger, which we do not test here. The program computing $a^\star$ also exposes the intermediate quantities it consumes, which we query separately, partitioning each instance into: all reads and answer correct; all reads correct but answer wrong (a reasoning failure); at least one read wrong (a perception failure, whatever the answer). The third cell is not a subset of the errors, since a wrong read can still yield a correct answer by coincidence.

%%%%%%%%%%%%%%%%%%%%%%%%%%%%%%%%%%%%%%%%%%%%%%%%%%%%%%%%%%%%%%%%%%%%%%%%%%%%%%%%
\section{What agreement measures}
\label{sec:theory}

Aggregation itself is standard: modal voting over i.i.d.\ answers concentrates exponentially, and the sample median converges to the median of the read distribution (\cref{lem:mode,lem:median}, \cref{app:standard}). Neither speaks to correctness, which is the relationship the rest of the paper is concerned with. The results below rest on two assumptions.

\begin{assumption}[Exchangeable renderings]
\label{ass:indep}
Conditional on $(D,q)$, the answers $\hat y_1,\dots,\hat y_K$ are i.i.d.\ draws from a distribution $P$ induced by the style distribution and the model. Where a mode is referenced we assume it is unique.
\end{assumption}

\begin{assumption}[Factored style space]
\label{ass:factored}
$\Theta = \Theta_1 \times \cdots \times \Theta_J$ is a product over $J$ nuisance factors and $\theta$ is drawn with independent coordinates. RENDEQ enforces this by construction (\cref{sec:benchmark}).
\end{assumption}

\Cref{ass:indep} is an idealization: it holds when style acts as independent noise on the read and fails when the model has a style-invariant bias, the regime \citet{estornell2024debate} analyze for correlated voters. \Cref{prop:lie} characterizes that failure mode in terms of quantities RENDEQ can estimate.

\subsection{Agreement estimates concentration, with a bias}
\label{sec:concentration}

Write $\pi = \max_a P(a)$ for the modal mass, $m^\star$ for the mode, and $Z = \mathbf{1}[m^\star = a^\star]$ for whether the mode is the correct answer.

\begin{proposition}[What $\rea$ estimates]
\label{prop:concentration}
Under \cref{ass:indep}, for categorical answers:
\begin{enumerate}[label=(\roman*)]
\itemsep0em
\item $\rea_{\text{cat}} \to \pi$ and $\mathbf{1}[\hat y_{\rmg} = a^\star] \to Z$ almost surely as $K\to\infty$. The limit of $\rea$ does not depend on $a^\star$.
\item $\mathbb{E}[\rea_{\text{cat}}] \ge \pi$ at every finite $K$, strictly unless $\pi = 1$.
\item The pairwise statistic $\rea_{\text{pair}} = \binom{K}{2}^{-1}\sum_{i<j}\mathbf{1}[\hat y_i = \hat y_j]$ is unbiased for the collision probability $A = \sum_a P(a)^2$, and $\pi^2 \le A \le \pi$.
\end{enumerate}
\end{proposition}

Part (ii) is a selection effect: $\rea_{\text{cat}}$ is the empirical \emph{maximum} of a multinomial frequency vector, so the winning answer is chosen partly by noise and its share is inflated, not negligibly at practical budgets (\cref{app:proof} works a numeric example), so any calibration of $\rea$ at small $K$ is optimistic by construction. Part (iii) supplies an unbiased alternative, which matters because our factor attribution (\cref{sec:stylefactors}) already uses a \emph{pairwise} rate while $\rea$ is modal; $\pi^2 \le A \le \pi$ is the conversion.

\subsection{Where agreement can lie}
\label{sec:wherelie}

Let $\beta = \max_{a \neq a^\star} P(a)$ be the mass of the leading wrong answer.

\begin{proposition}[Confident errors must be concentrated]
\label{prop:lie}
Under \cref{ass:indep}:
\begin{enumerate}[label=(\roman*)]
\itemsep0em
\item $\pi = \max\big(P(a^\star),\,\beta\big)$ and $Z = \mathbf{1}[P(a^\star) > \beta]$.
\item On instances with $Z = 0$, $\pi = \beta$: the agreement being reported \emph{is} the mass of a single wrong answer.
\item Call the errors \emph{diffuse at level} $L \ge 1$ if $\beta \le (1 - P(a^\star))/L$. Then $Z = 0$ implies $\pi \le 1/L$; equivalently, $\pi > 1/L$ certifies $Z = 1$.
\end{enumerate}
\end{proposition}

This is the precise form of the idiosyncratic-versus-systematic distinction the literature states informally. By (ii), agreement is high and wrong only when one specific wrong answer carries the agreement mass; dispersion spread over many wrong answers cannot produce a confidently wrong consensus. By (iii), diffuseness converts into a certifying threshold: with wrong mass over two alternatives, agreement above $1/2$ certifies the mode; over three, above $1/3$; at $L{=}1$, a single systematic misread, the threshold is $1$ and nothing certifies (\cref{fig:threshold}). Where \cref{cor:auroc} identifies the coupling as the quantity of interest, \cref{prop:lie} identifies the property of the model that supplies it.

$L$ references $a^\star$, so it cannot be estimated at use time. It can be estimated offline on any corpus with programmatic ground truth, which is exactly what RENDEQ supplies: measure $L$ on the generator, then apply the threshold $1/L$ label-free elsewhere. Measuring $L$ on the generator is therefore what licenses a threshold applied elsewhere.

We estimated $\hat L$ on every wrong-mode instance in the replicated pool and found it concentrated toward the low-diffuseness end (median $1.33$--$1.50$ across models, so the certifying threshold $1/\hat L$ is typically well above the midpoint); \cref{app:per-model} gives the full distribution and ties it to which families drive B1's remaining advantage over $\rea$ (\cref{tab:auroc}).

\begin{corollary}[The coupling is what matters]
\label{cor:auroc}
Let instances be drawn from a population inducing random variables $(\pi, Z)$. The limiting AUROC of $\rea$ as a predictor of $\rmg$'s correctness is $\Pr[\pi_1 > \pi_0] + \tfrac12\Pr[\pi_1 = \pi_0]$, where $\pi_1$ and $\pi_0$ are draws of $\pi$ given $Z{=}1$ and $Z{=}0$. If $\pi \perp Z$ it is $\tfrac12$, and there exist populations on which $\rea \to 1$ while $\rmg$ is wrong with probability $\to 1$.
\end{corollary}

Our AUROC values are computed at $K{=}8$, where $\rea$ estimates $\pi$ noisily and with bias, so they need not equal this limit; we do not assume a sign for the gap.

\subsection{Which perturbation carries the dispersion}
\label{sec:whichperturbation}

Let $\hat y(\theta)$ be the answer under style $\theta$, and define the pairwise disagreement $D = \Pr_{\theta,\theta'}[\hat y(\theta) \neq \hat y(\theta')]$ for independent $\theta,\theta'$. Let $D_j$ be the same quantity when only coordinate $j$ is resampled, and $D^{-j}$ when coordinate $j$ is held \emph{fixed} and shared.

\begin{proposition}[Factor decomposition]
\label{prop:factor}
Under \cref{ass:factored}, $D \le \sum_{j=1}^{J} D_j$. Equality requires both that the single-factor flip events along the hybrid chain be almost surely disjoint and that no flip be reversed by a later one; either failing makes the bound strict. Moreover $D - D_j \le D^{-j} \le D$ for every $j$.
\end{proposition}

The first bound licenses reading \cref{sec:stylefactors} as a decomposition rather than a list of correlations: a factor with small $D_j$ contributes at most $D_j$ however many levels it has, so perturbation \emph{identity} matters, not \emph{diversity} (the gap $\sum_j D_j - D$ upper-bounds interaction and cancellation rather than measuring them). The second bound predicts the ablation this paper most needs: holding one factor fixed cannot remove more than that factor's own contribution, so unless a single factor accounts for essentially all of $D$, holding the plotting library fixed should \emph{attenuate} the effect rather than eliminate it (\cref{sec:stylefactors} reports the $D_j$; the totals needed to make this quantitative are among the missing items there).

\begin{remark}[Small even $K$]
\label{rem:k2}
For $K{=}2$ the sample median is the mean of the two reads, so $\rmg$ can be wrong where a single rendering was right. Accuracy under median aggregation is not monotone in $K$ near $K{=}2$. This is a property of the estimator, not of render-equivalence.
\end{remark}

Together the three propositions give the four predictions tested in \cref{sec:experiments}: aggregation helps wherever dispersion exists (\cref{lem:mode,lem:median}); the perturbation with the largest $D_j$ wins at a matched budget (\cref{prop:factor}); agreement need not beat signals carrying evidence rather than dispersion (\cref{cor:auroc}); and an objective that raises $\pi$ while lowering diffuseness $L$ destroys its own certificate (\cref{prop:lie}). Proofs are in \cref{app:proof}.

%%%%%%%%%%%%%%%%%%%%%%%%%%%%%%%%%%%%%%%%%%%%%%%%%%%%%%%%%%%%%%%%%%%%%%%%%%%%%%%%
\section{RENDEQ}
\label{sec:benchmark}

\textbf{RENDEQ} is a generator and evaluation protocol built for control rather than scale: every instance carries exact ground truth and only cosmetic style varies within a set, so any answer change is attributable to style. Three properties make it an instrument for \cref{cor:auroc}: equivalence holds by construction, so $\pi$ is estimable without confounding; the key is programmatic, so $Z$ is observable; and factors are individually togglable, so each perturbation's contribution can be attributed.

\paragraph{Instance and axes.} An instance is a tuple $(D, q, a^\star, \{r_j\}, \{(\theta_i, x_i)\}_{i=1}^{K})$: data, question, exact answer, the exact intermediate reads consumed by the answer program, and $K$ renderings with logged style configurations (default $K{=}8$). Four axes vary independently (\cref{tab:bench}), fixing what is asked (plot family, question type), the style space $\Theta$ (style groups), and how hard the read is (difficulty controls).

\paragraph{Diagnostics and scope.} Beyond $\rea$ and $\rmg$, RENDEQ defines two reusable per-instance diagnostics: \textbf{render-flip rate}, $\mathrm{RFR} = \frac{1}{|\mathcal{D}|}\sum \mathbf{1}[|\{\hat y_i\}| > 1]$, which measures instability with no ground truth, and \textbf{render-equivalence robustness}, $\mathrm{RER} = \frac{1}{K}\sum_i \mathrm{acc}(\hat y_i)$, style-averaged accuracy, which needs ground truth and is for evaluation only. A majority of instances have at least one disagreeing rendering on every model (RFR $0.59$--$0.73$), and RFR is strongly negatively correlated with per-family accuracy ($r{=}{-0.69}$); \cref{app:rfr} gives the per-model and per-family numbers. All results here use a single in-style synthetic split; the generator also supports an out-of-style split and a real-data anchor built by re-rendering tables from a public benchmark, neither run for this paper, so external validity to real charts is untested (see Limitations).

\begin{table}[t]
\centering
\footnotesize
\setlength{\tabcolsep}{3pt}
\begin{tabular}{@{}l p{5.35cm}@{}}
\toprule
\kilth{Axis} & \kilth{Levels} \\
\midrule
Plot family & bar, grouped bar, stacked bar, line, scatter, pie, log-axis \\
\addlinespace[1pt]
Question type & value read, comparison, extremum, trend sign, arithmetic \\
\addlinespace[1pt]
Style $\Theta$ & library, palette, theme, gridlines, aspect ratio, marker, font, ticks, legend, resolution \\
\addlinespace[1pt]
Difficulty & \#series, \#points, precision, near-ties, overlap \\
\bottomrule
\end{tabular}
\caption{\textbf{RENDEQ axes.} Style groups are individually togglable so a single nuisance factor can be varied in isolation, which is what makes the attribution in \cref{sec:stylefactors} causal with respect to style.}
\label{tab:bench}
\end{table}

%%%%%%%%%%%%%%%%%%%%%%%%%%%%%%%%%%%%%%%%%%%%%%%%%%%%%%%%%%%%%%%%%%%%%%%%%%%%%%%%
\section{Experiments}
\label{sec:experiments}

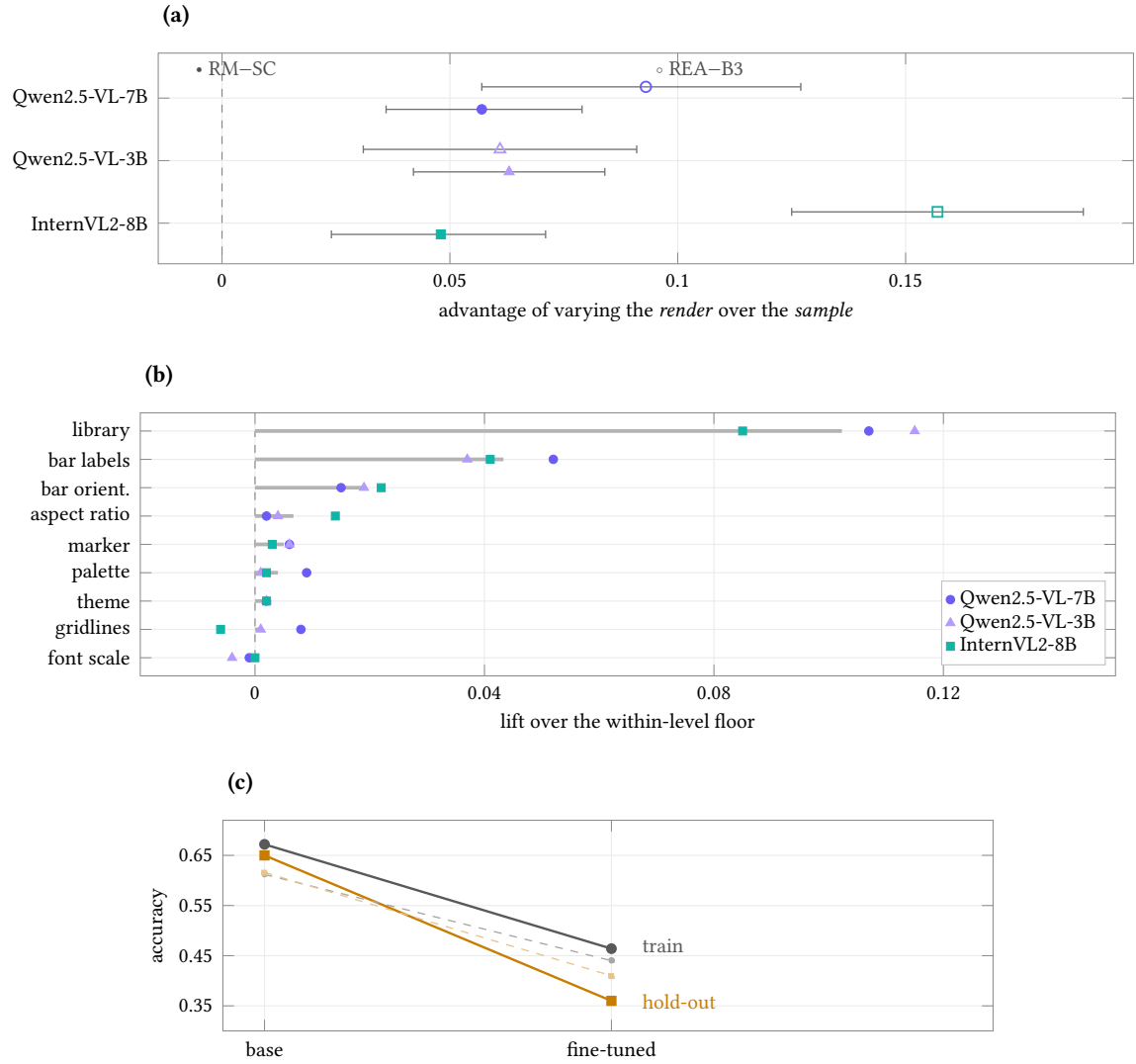
\begin{figure}[!t]
\centering

% ======================== Panel (a) ========================
\begin{tikzpicture}[font=\footnotesize]
\begin{axis}[
  width=0.96\columnwidth,
  height=4.4cm,
  xlabel={advantage of varying the \emph{render} over the \emph{sample}},
  xlabel style={font=\scriptsize, yshift=2pt},
  ymin=0.35, ymax=3.72,
  ytick={1,2,3},
  yticklabels={{InternVL2-8B},{Qwen2.5-VL-3B},{Qwen2.5-VL-7B}},
  y tick label style={font=\scriptsize},
  xmin=-0.014, xmax=0.200,
  xtick={0,0.05,0.10,0.15},
  xticklabel style={
    /pgf/number format/fixed,
    /pgf/number format/precision=2,
    font=\scriptsize
  },
  scaled x ticks=false,
  grid=major,
  grid style={black!8},
  axis line style={black!45},
  tick style={black!45},
  title={(a)},
  title style={
    font=\footnotesize\bfseries,
    at={(0,1)},
    anchor=south west,
    xshift=-2pt
  },
]
\draw[black!50, dashed]
  (axis cs:0,0.5) -- (axis cs:0,3.55);

% Filled markers: RM-SC accuracy
\draw[black!55, line width=0.5pt]
  (axis cs:0.036,2.82) -- (axis cs:0.079,2.82);
\draw[black!55, line width=0.5pt]
  ([yshift=-1.5pt]axis cs:0.036,2.82) --
  ([yshift=1.5pt]axis cs:0.036,2.82);
\draw[black!55, line width=0.5pt]
  ([yshift=-1.5pt]axis cs:0.079,2.82) --
  ([yshift=1.5pt]axis cs:0.079,2.82);

\draw[black!55, line width=0.5pt]
  (axis cs:0.042,1.82) -- (axis cs:0.084,1.82);
\draw[black!55, line width=0.5pt]
  ([yshift=-1.5pt]axis cs:0.042,1.82) --
  ([yshift=1.5pt]axis cs:0.042,1.82);
\draw[black!55, line width=0.5pt]
  ([yshift=-1.5pt]axis cs:0.084,1.82) --
  ([yshift=1.5pt]axis cs:0.084,1.82);

\draw[black!55, line width=0.5pt]
  (axis cs:0.024,0.82) -- (axis cs:0.071,0.82);
\draw[black!55, line width=0.5pt]
  ([yshift=-1.5pt]axis cs:0.024,0.82) --
  ([yshift=1.5pt]axis cs:0.024,0.82);
\draw[black!55, line width=0.5pt]
  ([yshift=-1.5pt]axis cs:0.071,0.82) --
  ([yshift=1.5pt]axis cs:0.071,0.82);

\addplot[
  only marks, mark=*, mark size=1.8pt, kilviolet
] coordinates {(0.057,2.82)};

\addplot[
  only marks, mark=triangle*, mark size=2.2pt, killilac
] coordinates {(0.063,1.82)};

\addplot[
  only marks, mark=square*, mark size=1.7pt, kilteal
] coordinates {(0.048,0.82)};

% Hollow markers: REA-B3 AUROC
\draw[black!55, line width=0.5pt]
  (axis cs:0.057,3.18) -- (axis cs:0.127,3.18);
\draw[black!55, line width=0.5pt]
  ([yshift=-1.5pt]axis cs:0.057,3.18) --
  ([yshift=1.5pt]axis cs:0.057,3.18);
\draw[black!55, line width=0.5pt]
  ([yshift=-1.5pt]axis cs:0.127,3.18) --
  ([yshift=1.5pt]axis cs:0.127,3.18);

\draw[black!55, line width=0.5pt]
  (axis cs:0.031,2.18) -- (axis cs:0.091,2.18);
\draw[black!55, line width=0.5pt]
  ([yshift=-1.5pt]axis cs:0.031,2.18) --
  ([yshift=1.5pt]axis cs:0.031,2.18);
\draw[black!55, line width=0.5pt]
  ([yshift=-1.5pt]axis cs:0.091,2.18) --
  ([yshift=1.5pt]axis cs:0.091,2.18);

\draw[black!55, line width=0.5pt]
  (axis cs:0.125,1.18) -- (axis cs:0.189,1.18);
\draw[black!55, line width=0.5pt]
  ([yshift=-1.5pt]axis cs:0.125,1.18) --
  ([yshift=1.5pt]axis cs:0.125,1.18);
\draw[black!55, line width=0.5pt]
  ([yshift=-1.5pt]axis cs:0.189,1.18) --
  ([yshift=1.5pt]axis cs:0.189,1.18);

\addplot[
  only marks, mark=o, mark size=2.0pt,
  kilviolet, line width=0.7pt
] coordinates {(0.093,3.18)};

\addplot[
  only marks, mark=triangle, mark size=2.4pt,
  killilac, line width=0.7pt
] coordinates {(0.061,2.18)};

\addplot[
  only marks, mark=square, mark size=1.9pt,
  kilteal, line width=0.7pt
] coordinates {(0.157,1.18)};

\node[black!70, font=\scriptsize, anchor=west]
  at (axis cs:-0.008,3.48) {$\bullet$\;$\rmg-$SC};

\node[black!70, font=\scriptsize, anchor=west]
  at (axis cs:0.093,3.48) {$\circ$\;$\rea-$B3};
\end{axis}
\end{tikzpicture}

\vspace{3mm}

% ======================== Panel (b) ========================
\begin{tikzpicture}[font=\footnotesize]
\begin{axis}[
  width=0.96\columnwidth,
  height=5.1cm,
  xlabel={lift over the within-level floor},
  xlabel style={font=\scriptsize, yshift=2pt},
  symbolic y coords={
    font scale,
    gridlines,
    theme,
    palette,
    marker,
    aspect ratio,
    bar orient.,
    bar labels,
    library
  },
  ytick=data,
  y tick label style={font=\scriptsize},
  xmin=-0.020, xmax=0.150,
  xtick={0,0.04,0.08,0.12},
  xticklabel style={
    /pgf/number format/fixed,
    /pgf/number format/precision=2,
    font=\scriptsize
  },
  scaled x ticks=false,
  enlarge y limits=0.08,
  grid=major,
  grid style={black!8},
  axis line style={black!45},
  tick style={black!45},
  legend style={
    at={(0.99,0.04)},
    anchor=south east,
    font=\scriptsize,
    draw=black!25,
    fill=white,
    row sep=-2pt,
    inner sep=1.6pt
  },
  legend cell align=left,
  title={(b)},
  title style={
    font=\footnotesize\bfseries,
    at={(0,1)},
    anchor=south west,
    xshift=-2pt
  },
]
\draw[black!50, dashed]
  (axis cs:0,font scale) -- (axis cs:0,library);

\draw[black!30, line width=1.4pt]
  (axis cs:0,library) -- (axis cs:0.1023,library);
\draw[black!30, line width=1.4pt]
  (axis cs:0,bar labels) -- (axis cs:0.0433,bar labels);
\draw[black!30, line width=1.4pt]
  (axis cs:0,bar orient.) -- (axis cs:0.0187,bar orient.);
\draw[black!30, line width=1.4pt]
  (axis cs:0,aspect ratio) -- (axis cs:0.0067,aspect ratio);
\draw[black!30, line width=1.4pt]
  (axis cs:0,marker) -- (axis cs:0.0050,marker);
\draw[black!30, line width=1.4pt]
  (axis cs:0,palette) -- (axis cs:0.0040,palette);
\draw[black!30, line width=1.4pt]
  (axis cs:0,theme) -- (axis cs:0.0020,theme);
\draw[black!30, line width=1.4pt]
  (axis cs:0,gridlines) -- (axis cs:0.0010,gridlines);
\draw[black!30, line width=1.4pt]
  (axis cs:0,font scale) -- (axis cs:-0.0017,font scale);

\addplot[
  only marks,
  mark=*,
  mark size=1.6pt,
  kilviolet
] coordinates {
  (0.107,library)
  (0.015,bar orient.)
  (0.002,aspect ratio)
  (0.006,marker)
  (0.052,bar labels)
  (0.009,palette)
  (0.008,gridlines)
  (-0.001,font scale)
  (0.002,theme)
};
\addlegendentry{Qwen2.5-VL-7B}

\addplot[
  only marks,
  mark=triangle*,
  mark size=2.0pt,
  killilac
] coordinates {
  (0.115,library)
  (0.019,bar orient.)
  (0.004,aspect ratio)
  (0.006,marker)
  (0.037,bar labels)
  (0.001,palette)
  (0.001,gridlines)
  (-0.004,font scale)
  (0.002,theme)
};
\addlegendentry{Qwen2.5-VL-3B}

\addplot[
  only marks,
  mark=square*,
  mark size=1.5pt,
  kilteal
] coordinates {
  (0.085,library)
  (0.022,bar orient.)
  (0.014,aspect ratio)
  (0.003,marker)
  (0.041,bar labels)
  (0.002,palette)
  (-0.006,gridlines)
  (0.000,font scale)
  (0.002,theme)
};
\addlegendentry{InternVL2-8B}
\end{axis}
\end{tikzpicture}

\vspace{3mm}

% ======================== Panel (c) ========================
\begin{tikzpicture}[font=\footnotesize]
\begin{axis}[
  width=0.78\columnwidth,
  height=4.4cm,
  ylabel={accuracy},
  ylabel style={font=\scriptsize, yshift=-3pt},
  xmin=-0.12, xmax=2.10,
  ymin=0.30, ymax=0.72,
  xtick={0,1},
  xticklabels={base,{fine-tuned}},
  ytick={0.35,0.45,0.55,0.65},
  yticklabel style={font=\scriptsize},
  xticklabel style={font=\scriptsize},
  grid=major,
  grid style={black!8},
  axis line style={black!45},
  tick style={black!45},
  title={(c)},
  title style={
    font=\footnotesize\bfseries,
    at={(0,1)},
    anchor=south west,
    xshift=-2pt
  },
]
\addplot[
  black!65,
  mark=*,
  mark size=1.6pt,
  line width=0.9pt
] coordinates {(0,0.672) (1,0.464)};

\addplot[
  kilamber,
  mark=square*,
  mark size=1.5pt,
  line width=0.9pt
] coordinates {(0,0.650) (1,0.360)};

\addplot[
  black!35,
  mark=*,
  mark size=1.1pt,
  line width=0.5pt,
  dashed
] coordinates {(0,0.612) (1,0.440)};

\addplot[
  kilamber!45,
  mark=square*,
  mark size=1.0pt,
  line width=0.5pt,
  dashed
] coordinates {(0,0.616) (1,0.410)};

\node[black!65, font=\scriptsize, anchor=west]
  at (axis cs:1.06,0.470) {train};

\node[kilamber, font=\scriptsize, anchor=west]
  at (axis cs:1.06,0.358) {hold-out};
\end{axis}
\end{tikzpicture}

\caption{\textbf{The effect is attributable, concentrated, and not safe to optimize.}
(a) At a matched $K{=}8$ budget, pooled across three instantiations on the corrected pipeline (\cref{app:replication}), varying the rendering outperforms varying the decoding sample for every model and both uses: $+4.8$ to $+6.3$ accuracy points for $\rmg{-}$SC and $+0.06$ to $+0.16$ AUROC for $\rea{-}$B3. Bars show 95\% cluster-bootstrap intervals based on 5,000 resamples; filled markers denote accuracy and hollow markers denote reliability.
(b) Dispersion is concentrated rather than distributed across cosmetic styles. The plotting-library effect is more than twice the next-largest effect and approximately an order of magnitude above the noise floor. Stems indicate across-model means, with rows ordered accordingly.
(c) Results for Qwen2.5-VL-7B, pooled across a five-run replication. Color identifies the split rather than the model. Fine-tuning on the model's own $K{=}8$ majority increases hold-out agreement in every run while reducing hold-out accuracy in every run.}
\label{fig:results}
\end{figure}

\begin{table}[t]
\centering
\caption{\textbf{The paper's four predictions,} each evaluated across three independent RENDEQ instantiations (\cref{sec:experiments}). $^\dagger$See \cref{sec:coupling} for model-level results and \cref{app:replication} for the earlier retracted reversal.}
\label{tab:predictions}
\scriptsize
\setlength{\tabcolsep}{2.2pt}
\renewcommand{\arraystretch}{1.05}
\resizebox{\columnwidth}{!}{%
\begin{tabular}{@{}p{0.27\columnwidth} p{0.17\columnwidth} p{0.22\columnwidth} p{0.27\columnwidth}@{}}
\toprule
\kilth{Prediction} & \kilth{Test} & \kilth{Outcome} & \kilth{Replicated?} \\
\midrule
Aggregation helps
& \cref{sec:aggregation}
& 3/3 models
& Yes, 3 instantiations \\

Render $>$ sample, accuracy
& \cref{sec:aggregation}
& 3/3 models
& Yes, 3 instantiations \\

Render $>$ sample, reliability
& \cref{sec:coupling}
& 3/3 models
& Yes, 3 instantiations \\

Agreement vs.\ evidence
& \cref{sec:coupling}
& \emph{Rejected}, 2/3$^\dagger$
& Yes, 3 instantiations \\

Consensus training inverts
& \cref{sec:selftraining}
& Accuracy drop, 5/5 runs
& Yes in direction; magnitude unstable \\
\bottomrule
\end{tabular}%
}
\end{table}

\paragraph{Setup.}
\label{sec:setup}
Every headline number here (\cref{tab:auroc,tab:accuracy}) is pooled over \textbf{three independently generated RENDEQ instantiations} (generator seeds 0, 1, 2; $N{=}350$ instances each, 7 families, $K{=}8$ renderings, no shared instances), replacing an earlier single-instantiation design (\cref{app:replication}). ``Seed'' always means a generator instantiation, not a decoding seed; decoding is greedy (temperature 0) in the main runs, deterministic and contributing no additional variance. Models are Qwen2.5-VL-7B-Instruct, Qwen2.5-VL-3B-Instruct, and InternVL2-8B \citep{bai2025qwen25vl,chen2024internvl2}, unchanged throughout (GPU and quantization settings in \cref{app:replication}). Of the 3150 pooled instances, 930 ($29.5\%$) are numeric (\cref{lem:median}), the rest categorical (\cref{lem:mode}). All three models date to 2024 or early 2025 and two share a family (Qwen2.5-VL); we did not add a current-generation model from a third family (see Limitations).

\paragraph{Metrics and baselines.} All $\rea$ values below are modal, upward-biased at finite $K$ (\cref{prop:concentration}(ii)); the unbiased pairwise alternative $\rea_{\text{pair}}$ gives the same ranking against B1 (\cref{app:kablation}). Accuracy uses a $\tau{=}5\%$ relative tolerance for numeric answers and exact match for categorical; $\tau$ is load-bearing in magnitude and, on one model, in direction too (\cref{app:kablation}). We report AUROC against correctness, risk--coverage AUC, and ECE after min-max normalizing each signal; bootstrap intervals use 5000 resamples unless noted, and paired accuracy comparisons use McNemar's test. Baselines: B1 mean token log-probability; B2 verbalized confidence; B3 response self-consistency over $K{=}8$ samples at a fixed rendering \citep{wang2023selfconsistency}; B4 the consistency signal of \citet{khan2024consistency}; B5 a prompt ensemble over three phrasings. B3 is the mechanism control, holding the image fixed and varying only decoding, so the $\rea$/B3 and $\rmg$/B3 contrasts isolate render-induced from sampling-induced dispersion at equal call count. We did not run semantic entropy, P(True), or an approximate-perturbation arm (see Limitations).

\subsection{Measuring the coupling}
\label{sec:coupling}

\begin{table}[h]
\centering
\footnotesize
\setlength{\tabcolsep}{2pt}
\begin{tabular}{@{}l S[table-format=1.3] S[table-format=1.3] S[table-format=1.3]@{}}
\toprule
\kilth{Signal} & {\kilth{7B}} & {\kilth{3B}} & {\kilth{IVL-8B}} \\
\midrule
B1 token logprob & 0.771 & 0.867 & 0.875 \\
\addlinespace[1.5pt]
\textbf{$\rea$ (ours)} & \textbf{0.864} & \textbf{0.858} & \textbf{0.907} \\
\addlinespace[1.5pt]
$\Delta$ ($\rea{-}$B1)
  & \multicolumn{1}{c}{$+0.094$}
  & \multicolumn{1}{c}{$-0.009$}
  & \multicolumn{1}{c}{$+0.033$} \\
95\% CI
  & \multicolumn{1}{c}{\tiny$[{+}.066,{+}.121]$}
  & \multicolumn{1}{c}{\tiny$[{-}.030,{+}.011]$}
  & \multicolumn{1}{c}{\tiny$[{+}.013,{+}.053]$} \\
\bottomrule
\end{tabular}
\caption{\textbf{$\rea$ beats B1}, mean token log-probability, on AUROC against correctness on two of three models, and is statistically tied with it on the third; pooled over three independently generated RENDEQ instantiations (seeds 0, 1, 2; $N{=}1050$ per model, $K{=}8$; cluster bootstrap over instances). B2--B5 are pooled the same way; see \cref{app:replication} for the full table.}
\label{tab:auroc}
\tabnote{$\rea$'s confidence interval excludes zero in its own favor on 7B and InternVL2-8B; on 3B the interval includes zero, a statistical tie. All $\rea$ values are modal and upward-biased at finite $K$ (\cref{prop:concentration}(ii)); \cref{app:replication} shows the ranking is unchanged under the unbiased pairwise alternative.}
\end{table}

\paragraph{Re-rendering versus resampling.} $\rea$ beats B3, the fixed-image self-consistency control, on all three models: pooled across the three independently generated instantiations used for \cref{tab:auroc} ($N{\approx}3150$), $\Delta = +0.102$, CI $[+0.083,+0.121]$, $p<10^{-4}$ (\cref{app:replication} gives the per-model breakdown, $+0.061$ to $+0.157$). This attributes the reliability signal to re-rendering specifically, rather than to generic answer instability, corroborating \citet{kaya2026efficient} in a setting where the augmentation is exactly meaning-preserving. The accuracy comparison behind this prediction (\cref{sec:aggregation}) uses $\rmg$, not $\rea$, and agrees in direction: both halves of the render-versus-sample prediction now point the same way.

\paragraph{Agreement versus evidence.} Against B1, mean token log-probability, the picture is genuinely mixed rather than uniform in either direction: pooled over three independently generated RENDEQ instantiations, $\rea$ beats B1 on Qwen2.5-VL-7B ($+0.094$, CI $[+0.066,+0.121]$) and InternVL2-8B ($+0.033$, CI $[+0.013,+0.053]$), and the two are statistically tied on Qwen2.5-VL-3B ($-0.009$, CI $[-0.030,+0.011]$, includes zero). We do not claim agreement dominates evidence-carrying signals in general (the prediction that motivated this comparison, \cref{sec:theory}, expected B1 to win, and it does not, on the models where the comparison is decisive), and the ranking is stable across the numeric-tolerance parameter $\tau$ on two of three models (\cref{app:kablation}). This is \cref{cor:auroc} in data: $\rea$ estimates dispersion and cannot carry evidence, while a token log-probability can, and here it does. An earlier version of this comparison, computed on a rendering pipeline where a silent library-export failure had collapsed all three "independent instantiations" to matplotlib-only images, reported the opposite direction on every model; \cref{app:replication} describes that bug, why it produced this specific reversal, and how we found it.

$\rea$ beating B1 head-to-head on two of three models does not mean B1 carries no information $\rea$ lacks: the two signals make partly independent errors on every model, and a cross-validated logistic combination of them beats B1 alone on every model, including 3B (\cref{app:per-model} gives the AUROC gains and CIs).

\subsection{Aggregation}
\label{sec:aggregation}

\begin{table}[h]
\centering
\footnotesize
\setlength{\tabcolsep}{2pt}
\begin{tabular}{@{}l S[table-format=1.3] S[table-format=1.3] S[table-format=1.3] S[table-format=+1.3]@{}}
\toprule
\kilth{Model} & {\kilth{single}} & {\kilth{SC}} & {\kilth{$\rmg$}} & {\kilth{$\rmg{-}$SC}} \\
\midrule
Qwen2.5-VL-7B & 0.644 & 0.644 & 0.701 & +0.057 \\
Qwen2.5-VL-3B & 0.571 & 0.563 & 0.626 & +0.063 \\
InternVL2-8B  & 0.563 & 0.576 & 0.624 & +0.048 \\
\bottomrule
\end{tabular}
\caption{\textbf{$\rmg$ beats both} single-render decoding and self-consistency (SC) on every model at a matched $K{=}8$ budget, pooled over three RENDEQ instantiations ($N{=}1050$ per model). SC: response self-consistency at fixed image, temperature 0.7, majority vote over 8 samples. 95\% CIs in \cref{app:per-model}.}
\label{tab:accuracy}
\end{table}

Aggregation holds: $\rmg$ improves on single-render decoding by 5.5 to 6.1 points, and on self-consistency (SC) by 4.8 to 6.3 points, on every model, pooled across three independent instantiations, every interval excluding zero. SC's own advantage over plain single-render decoding is smaller and less consistent in sign than re-rendering's advantage under $\rmg$ (\cref{app:per-model}), so the render-versus-sample accuracy claim still holds; one caveat is that the main $\rmg$ runs decode at temperature 0 while SC needs 0.7 to vary at all, so the contrast is at equal call count but not equal decoding configuration. This table omits \citeauthor{jiang2025chartcoca}'s method, the closest label-free test-time approach for charts (\cref{tab:novelty}), and token-level aggregation, reported as stronger by \citet{kaya2026efficient} (see Limitations). The intermediate-read query defined in \cref{sec:method} decomposes these errors into perception and reasoning failures (\cref{app:per-model}, \cref{sec:errors}): roughly four fifths of instances have at least one bad intermediate read, and true reasoning failures with every read correct are rare.

\subsection{Identity, not diversity}
\label{sec:stylefactors}

For each style factor and instance we compute the cross-level disagreement rate, the fraction of render pairs where the factor level differs and the parsed answers differ, minus the within-level rate as a noise floor. Pooled over the three replicated instantiations and all three models ($N{\approx}3150$ render pairs per factor; cluster bootstrap over instances, 3000 resamples):

\begin{center}
\footnotesize
\setlength{\tabcolsep}{4pt}
\begin{tabular}{@{}l S[table-format=+1.3] l@{}}
\toprule
\kilth{Factor} & {\kilth{Lift}} & \kilth{95\% CI} \\
\midrule
library      & +0.102 & $[+0.084,+0.120]$$^\ast$ \\
bar\_labels  & +0.043 & $[+0.025,+0.060]$$^\ast$ \\
bar\_orient  & +0.019 & $[+0.001,+0.036]$$^\ast$ \\
tick\_density & +0.009 & $[-0.008,+0.026]$ \\
pie\_start\_angle & +0.009 & $[-0.009,+0.026]$ \\
(9 further factors) & {$\le+0.007$} & all include zero \\
\bottomrule
\end{tabular}
\end{center}

\noindent The plotting library, matplotlib against plotly, dominates: its lift is more than twice the second-largest factor's and its interval clears zero by a wide margin. \texttt{bar\_labels} and \texttt{bar\_orient} have small but real lifts (both intervals exclude zero, though barely for the latter); every other factor's interval includes zero, so we do not claim they contribute nothing, only that this design cannot distinguish their contribution from noise. By \cref{prop:factor} these lifts bound each factor's contribution to the total: what matters is not how \emph{many} perturbations are applied but \emph{which}, and a practitioner on a budget should vary the backend before anything else. It is also the finding that most constrains our claim, since switching backend is less purely cosmetic than a palette change. This table could not even be computed on an earlier, buggy version of the pipeline, and the mechanism-isolating ablation it would motivate (library held fixed) remains genuinely missing, too few instances have all eight renderings in one library to power it (\cref{app:replication}).

\subsection{Consensus self-training can invert}
\label{sec:selftraining}

\citet{kaya2026efficient} fine-tune on consensus pseudo-labels derived from augmented inputs and report gains across nine benchmarks. \Cref{prop:concentration} says such an objective raises $\pi$ without moving $m^\star$, so it should help only where the coupling is already strong. We ran the same recipe on RENDEQ: LoRA ($r{=}8$, 3 epochs) on Qwen2.5-VL-7B with the model's own $K{=}8$ majority as the label and no ground truth, training on bar, grouped bar, stacked bar, line, scatter (250 instances) and holding out pie and log-axis (100 instances).

We ran this recipe five times: three decoding seeds on one dataset instantiation, plus one seed each on the two other independently generated instantiations used elsewhere in this paper (\cref{app:replication}). Accuracy fell under $\rmg$ in every run, by 3.6 to 21.2 points on training families and 10.4 to 27.2 on held-out ones; under single-render decoding, held-out accuracy fell in all five runs (8.4 to 18.8 points), while training-family accuracy fell in three of five and rose slightly in two ($+0.6$, $+1.0$ points). The held-out collapse is more consistent and more severe, and it is also the split with no ground-truth-adjacent signal at fine-tuning time, so it is the one the analysis is really about: on families where the base majority was already wrong, notably scatter, the objective trained the model to hold the wrong answer more tightly rather than to correct it, though the size of that effect is instantiation-dependent (\cref{app:replication}). We did not replicate \citet{kaya2026efficient} (they augment natural images and aggregate at the token level, we re-render figures and aggregate at the answer level, and the models differ), but what the two share is the objective, which \cref{prop:lie} predicts is regime-dependent. $\rea$'s AUROC change tracks this split and is consistent within each: it rose in every run on held-out families (0.009 to 0.100), the model growing more confidently self-consistent on families it now answers worse; it fell in every run on training families (0.014 to 0.124), the opposite of what fine-tuning on a family's own majority vote should do if the sharpening tracked correctness.

%%%%%%%%%%%%%%%%%%%%%%%%%%%%%%%%%%%%%%%%%%%%%%%%%%%%%%%%%%%%%%%%%%%%%%%%%%%%%%%%
\section{Conclusion}
\label{sec:conclusion}

Across four checks, the coupling between agreement and correctness turned out to be real, measurable, and neither as strong nor as fragile as it first looked. Render marginalization beats every control we matched it against, on every model and every replicated instantiation: the one result here that never moved, not under replication and not under the rendering-pipeline bug that moved almost everything else. Cross-render agreement, a purely label-free signal, beats an evidence-carrying baseline on two of three models rather than losing outright, visible only once the render-equivalence perturbation was actually applied: a reminder that a measurement instrument built to isolate a quantity can still be defeated by an implementation detail far from the theory it tests. And exactly as \cref{prop:lie} predicts, optimizing for agreement directly breaks the coupling rather than exploiting it: fine-tuning on a model's own consensus makes it agree with itself more and know less, inverting a result reported elsewhere rather than merely failing to reproduce it. Together these argue for treating concentration and correctness as related but separable quantities, not proxies for one another: a separation render-equivalence sets make directly measurable, so the same instrument that lets a reviewer check our numbers also checks whether agreement is safe to optimize for. Here, it is not.

%%%%%%%%%%%%%%%%%%%%%%%%%%%%%%%%%%%%%%%%%%%%%%%%%%%%%%%%%%%%%%%%%%%%%%%%%%%%%%%%
%% BACK MATTER
%%%%%%%%%%%%%%%%%%%%%%%%%%%%%%%%%%%%%%%%%%%%%%%%%%%%%%%%%%%%%%%%%%%%%%%%%%%%%%%%
\section*{Limitations}
\addcontentsline{toc}{section}{Limitations}

Every number here comes from a single in-style synthetic split of our own generator. The out-of-style split and the real-data anchor were not run, so external validity to real charts is untested, and the results characterize behaviour on RENDEQ rather than on chart question answering generally. Our negative result for consensus self-training is likewise measured on one generator, one model, and one LoRA configuration; we checked it across three decoding seeds and three dataset instantiations (\cref{app:replication}) and the direction of the accuracy drop held in every run, but its magnitude did not, so the LoRA configuration, model, and generator remain unvaried. It is not a replication of \citet{kaya2026efficient}, whose augmentations, aggregation level, models and benchmarks all differ, and it should be read as identifying a regime in which the objective fails rather than as evidence against their finding.

The experimental programme is incomplete in ways we state rather than leave to the reader to notice. Three arms are absent and each bears on a claim we make. There is no \emph{approximate}-perturbation arm: B3 separates re-rendering from decoding noise, but nothing separates exact equivalence from ordinary pixel augmentation, which is the delta this paper argues for, so that delta is currently reasoned rather than measured. The baseline set omits semantic entropy \citep{kuhn2023semantic,farquhar2024semantic} and P(True), the two standard label-free signals, and a direct run of \citet{zhang2024vluncertainty}. And \citet{jiang2025chartcoca}, the closest label-free test-time method for charts, is discussed but not run against $\rmg$.

Four further gaps are smaller but real. Our models date to 2024 and early 2025 and two of the three share a family. The tolerance $\tau$ is load-bearing throughout and never varied. We report threshold-free summaries, AUROC and risk--coverage area, but no accuracy at fixed coverage and no held-out threshold selection, which is what a deployed abstention rule would need. We make roughly eighteen paired comparisons across signals, models and metrics without correcting for multiple testing. Finally, the perception-versus-reasoning split in \cref{app:per-model} (\cref{sec:errors}) rests entirely on a separate intermediate-read query whose own reliability we never validate.

The analysis assumes a unique mode where one is referenced, and \cref{prop:lie}(iii) is stated in terms of a diffuseness level $L$ that we have not yet measured. \Cref{ass:indep} treats renderings as i.i.d.\ draws. Real style spaces are finite and correlated, and \cref{sec:stylefactors} shows one factor dominates, so the assumption is a working idealization whose failure mode is the bias regime the propositions identify. Until the library-held-fixed ablation is run we cannot separate sensitivity to cosmetic style in general from sensitivity to the plotting backend, and that ablation could narrow the paper's scope.

We aggregate at the answer level, while \citet{kaya2026efficient} report token-level aggregation is stronger; we have not tested whether that changes any conclusion. $\rmg$ cannot fix systematic, style-invariant errors and costs $K\times$ inference.

Agreement beats token log-probability on two of the three models we tested and is statistically tied with it on the third, replicated across three independently generated instantiations (\cref{app:replication}). This replication went through an intermediate stage, reported in full in that appendix, in which a rendering-pipeline bug silently removed the render diversity the comparison depends on and produced the opposite conclusion on every model; we found and fixed that bug before finalizing the numbers reported here. We do not claim agreement is a universally stronger label-free signal (the comparison is a tie on one of three models, and the theory, \cref{sec:theory}, predicts model- and data-dependence rather than a universal ranking), but the direction of the evidence, once the rendering pipeline actually renders what it is asked to, favours agreement more often than not on this benchmark. The K-ablation in \cref{app:kablation} and the calibration numbers in \cref{app:calibration} have both been recomputed on the corrected three-instantiation pool, and the apparent denominator mismatch in the error-decomposition breakdown has been traced to a labeling ambiguity rather than a data error (\cref{sec:errors}).

\begin{availability}
The generator, the agreement and aggregation code, and the evaluation scripts are available now at \url{https://github.com/KurbanIntelligenceLab/rendeq}, with pinned versions and seeds. The RENDEQ data (generated instances, manifests, and rendered images) are available at \url{https://huggingface.co/datasets/rasulkhanbayov/rendeq} under CC-BY-4.0; the generator, agreement/aggregation code, and evaluation scripts under the MIT license.
% TODO at camera-ready: swap the anonymized link above for the de-anonymized
% repository URL; complete the ARR Responsible NLP Research checklist against
% what was actually run and submit it alongside the paper.
\end{availability}

%% This work uses publicly available models and synthetic, programmatically
%% generated figures. No human subjects, personal data, or private data are
%% involved. The main risk is misplaced trust: \cref{cor:auroc} shows high
%% agreement on a systematically misread figure is exactly what this signal
%% cannot detect, so deploying $\rea$ as a safety gate without that caveat
%% would be unsafe. We state it in the analysis, the discussion, and the
%% limitations. AI tools assisted in drafting the text; the human authors are
%% the only authors, verified the mathematics and the method, and take full
%% responsibility for all content, including every claim, number, and citation.
%% (This paragraph was the ACL-template "Ethics Statement" section; the
%% template has no equivalent named block, so its content lives here as a
%% comment pending a decision on where it should surface -- see the porting
%% notes at the end of this file.)

%%%%%%%%%%%%%%%%%%%%%%%%%%%%%%%%%%%%%%%%%%%%%%%%%%%%%%%%%%%%%%%%%%%%%%%%%%%%%%%%
\bibliography{references}

%%%%%%%%%%%%%%%%%%%%%%%%%%%%%%%%%%%%%%%%%%%%%%%%%%%%%%%%%%%%%%%%%%%%%%%%%%%%%%%%
\appendix
\section{Positioning}
\label{app:positioning}

\begin{table}[h]
\centering
\footnotesize
\setlength{\tabcolsep}{3pt}
\begin{tabular}{@{}l c c c@{}}
\toprule
\kilth{Work} & \kilth{perturbs} & \kilth{equiv.} & \kilth{exact} \\
              & \kilth{input}    & \kilth{guar.}  & \kilth{key}   \\
\midrule
\citet{mukhopadhyay2024robustcqa} & yes & by design & no \\
\citet{shin2025losingplot}        & yes & no        & no \\
\citet{wang2023selfconsistency}   & no  & n/a       & no \\
\citet{khan2024consistency}       & no  & n/a       & no \\
\citet{jiang2025chartcoca}        & no  & n/a       & synth. \\
\citet{zhang2024vluncertainty}    & yes & no        & no \\
\citet{rosenfeld2025stability}    & yes & no        & no \\
\citet{kaya2026efficient}         & yes & no        & no \\
\midrule
\textbf{This work} & \textbf{yes} & \textbf{yes} & \textbf{yes} \\
\bottomrule
\end{tabular}
\caption{\textbf{The closest work.} ``Equivalence guaranteed'' asks whether the perturbation is meaning-preserving by construction or only by assumption; ``exact key'' whether a programmatic answer exists. Only when both hold can concentration and correctness be measured separately.}
\label{tab:novelty}
\tabnote{Every row was checked against its source.}
\end{table}

\section{Standard aggregation results}
\label{app:standard}

These two are classical and are used in the main body without comment. We state them for completeness and claim no novelty for either.

\begin{lemma}[Modal vote]
\label{lem:mode}
Let answers be categorical, let \cref{ass:indep} hold, and write $p = \Pr[\hat y_i \neq a^\star] < \tfrac12$. Under any tie-breaking rule,
\begin{equation*}
\begin{split}
\Pr[\hat y_{\rmg} \neq a^\star] \;&\le\; \Pr[\mathrm{Bin}(K,p) \ge \tfrac{K}{2}]\\
&\le\; e^{-2K(\frac12 - p)^2}.
\end{split}
\end{equation*}
When the answer space has two elements and $K$ is odd the first inequality is an equality, and the middle quantity is strictly decreasing along odd $K$.
\end{lemma}

\begin{proof}
Let $W = |\{i : \hat y_i \neq a^\star\}| \sim \mathrm{Bin}(K,p)$. If $W < K/2$ then $a^\star$ is returned by $K-W > K/2$ renderings, so it holds a strict majority and is the unique mode; any tie-breaking rule, deterministic or randomized, then returns $a^\star$. Hence $\{\hat y_{\rmg} \neq a^\star\} \subseteq \{W \ge K/2\}$, which uses identical distribution only. With $t = \tfrac12 - p > 0$, $\{W \ge K/2\} = \{W - Kp \ge Kt\}$ and Hoeffding's inequality for a sum of $K$ independent $[0,1]$-valued variables gives $\Pr[W - Kp \ge Kt] \le e^{-2Kt^2}$; independence is used here and only here. With answer space $\{a^\star,b\}$ and $K$ odd, $W \neq K/2$ always and $\hat y_{\rmg} = b$ iff $W > K/2$, so the containment is an equality. Monotonicity along odd $K$ is the Condorcet jury theorem; it does not hold over all $K$, since even $K$ admits ties and the sequence oscillates, which is why the statement is restricted.
\end{proof}

\begin{lemma}[Median aggregation]
\label{lem:median}
Let answers be continuous, let \cref{ass:indep} hold with $P$ having a unique median $\mu$ and a density $\varphi$ positive and continuous at $\mu$. Then $\hat y_{\rmg} \to \mu$ almost surely and $\sqrt{K}(\hat y_{\rmg} - \mu) \Rightarrow \mathcal{N}(0, \tfrac{1}{4\varphi(\mu)^2})$.
\end{lemma}

\begin{proof}
Standard sample-median consistency and asymptotic normality.
\end{proof}

The substantive point for this paper is the identity of the limit. Median aggregation suppresses dispersion at rate $K^{-1/2}$ but converges to $\mu$, the centre of the model's style-induced read distribution, not to $a^\star$. If a bias survives re-rendering, so $\mu \neq a^\star$, no $K$ makes $\rmg$ correct under a tolerance $\tau < |\mu - a^\star|/|a^\star|$. Aggregation removes variance, never bias, and \cref{lem:median} is the case the categorical self-consistency analyses do not cover even though most of our questions fall into it.

\section{Proofs of the main results}
\label{app:proof}

\subsection{Proof of \texorpdfstring{\cref{prop:concentration}}{Proposition 1}}

\paragraph{(i).} Let $\hat P_K$ be the empirical distribution of $\{\hat y_i\}_{i=1}^K$. By the strong law applied to each answer value, $\hat P_K(a) \to P(a)$ almost surely for every $a$. Since the mode $m^\star$ is unique there is almost surely a finite $K_0$ beyond which $\arg\max_a \hat P_K(a) = m^\star$, so $\hat y_{\rmg} \to m^\star$ and $\mathbf{1}[\hat y_{\rmg} = a^\star] \to \mathbf{1}[m^\star = a^\star] = Z$. On the same event $\rea_{\text{cat}} = \hat P_K(m^\star) \to P(m^\star) = \pi$. Neither the definition of $\pi$ nor the convergence uses $a^\star$.

\paragraph{(ii).} For any fixed $a_0$, $\max_a \hat P_K(a) \ge \hat P_K(a_0)$ pointwise, so $\mathbb{E}[\rea_{\text{cat}}] \ge \mathbb{E}[\hat P_K(a_0)] = P(a_0)$. Taking $a_0 = m^\star$ gives $\mathbb{E}[\rea_{\text{cat}}] \ge \pi$. Equality requires $\hat P_K(m^\star) = \max_a \hat P_K(a)$ almost surely. If $\pi < 1$ there is some $b \neq m^\star$ with $P(b) > 0$, and the event that all $K$ draws equal $b$ has probability $P(b)^K > 0$; on it $\hat P_K(b) = 1 > \hat P_K(m^\star)$, so the inequality is strict. If $\pi = 1$ then $\rea_{\text{cat}} \equiv 1 = \pi$. Numerically, for $P=(0.5,0.2,0.2,0.1)$ the bias $\mathbb{E}[\rea_{\text{cat}}]-\pi$ is $+0.17$ at $K{=}2$ and $+0.04$ at $K{=}8$: not negligible at the budgets used throughout this paper.

\paragraph{(iii).} $\rea_{\text{pair}}$ averages $\binom{K}{2}$ indicators, each with expectation $\Pr[\hat y_i = \hat y_j] = \sum_a P(a)^2 = A$ for $i \neq j$ by independence and identical distribution, so $\mathbb{E}[\rea_{\text{pair}}] = A$ regardless of $K$. For the sandwich, $A = \sum_a P(a)^2 \ge P(m^\star)^2 = \pi^2$ since $\pi^2$ is one of the summands, and $A = \sum_a P(a)P(a) \le \pi \sum_a P(a) = \pi$. \hfill$\square$

\begin{figure}[h]
\centering
\begin{tikzpicture}[font=\footnotesize]
\begin{axis}[
  width=0.98\columnwidth, height=4.2cm,
  xlabel={concentration $\pi$}, ylabel={density},
  xlabel style={font=\footnotesize, yshift=2pt}, ylabel style={font=\footnotesize, yshift=-6pt},
  xmin=0, xmax=1, ymin=0, ymax=1.42,
  xtick={0,0.33,0.5,1}, xticklabels={$0$,$\tfrac13$,$\tfrac12$,$1$},
  ytick=\empty, axis line style={black!45}, tick style={black!45},
  xticklabel style={font=\footnotesize},
]
% certified region
\addplot[draw=none, fill=kilteal!10] coordinates {(0.5,0) (1,0) (1,1.42) (0.5,1.42)} \closedcycle;
% Z = 0 : confined to pi <= 1/L
\addplot[kilamber, line width=0.9pt, smooth, samples=140, domain=0.02:0.50]
  {1.05*exp(-((x-0.36)^2)/0.013)};
% Z = 1 : concentrated high
\addplot[kilviolet, line width=0.9pt, smooth, samples=140, domain=0.10:1]
  {0.72*exp(-((x-0.60)^2)/0.08)};
\draw[black!60, dashed, line width=0.7pt] (axis cs:0.5,0) -- (axis cs:0.5,1.30);
\node[kilamber, font=\scriptsize, anchor=south] at (axis cs:0.27,1.09) {$Z{=}0$};
\node[kilviolet, font=\scriptsize, anchor=south] at (axis cs:0.60,0.78) {$Z{=}1$};
\node[font=\scriptsize, black!65] at (axis cs:0.76,1.24) {certified};
\node[font=\scriptsize, anchor=east, black!65] at (axis cs:0.488,1.20) {$1/L$};
\end{axis}
\end{tikzpicture}
\caption{\textbf{Why diffuseness sets the threshold.} Schematic distributions of concentration $\pi$ over instances whose modal answer is correct ($Z{=}1$) and incorrect ($Z{=}0$). \Cref{prop:lie}(iii) confines all $Z{=}0$ mass to $\pi \le 1/L$, so agreement above that threshold certifies the mode; the shaded band is the certified region for $L{=}2$. \Cref{cor:auroc} reads the same picture as a ranking problem: the AUROC of $\rea$ is the probability that a $Z{=}1$ instance is more concentrated than a $Z{=}0$ one, so it degrades as the two distributions approach each other and equals $\tfrac12$ when they coincide. Shapes are illustrative, calibrated to the measured AUROC.}
\label{fig:threshold}
\end{figure}
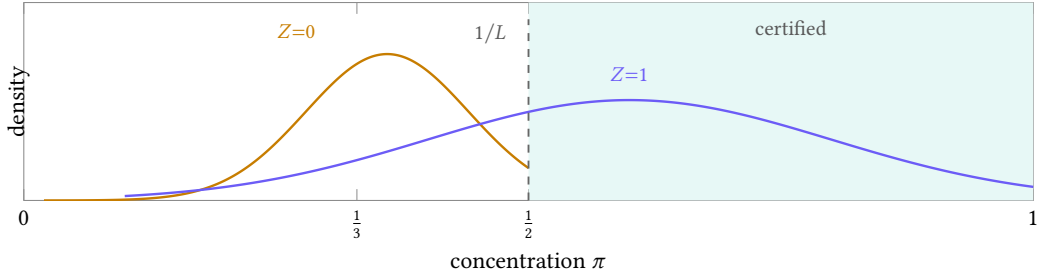

\subsection{Proof of \texorpdfstring{\cref{prop:lie}}{Proposition 2}}

\paragraph{(i).} The maximum over all answers splits into the correct answer and the rest, $\pi = \max_a P(a) = \max\big(P(a^\star), \max_{a \neq a^\star} P(a)\big) = \max(P(a^\star), \beta)$. The mode is $a^\star$ exactly when $P(a^\star)$ strictly exceeds every other mass, that is $P(a^\star) > \beta$, so $Z = \mathbf{1}[P(a^\star) > \beta]$.

\paragraph{(ii).} If $Z = 0$ then $P(a^\star) \le \beta$, so the maximum in (i) is attained by $\beta$ and $\pi = \beta$.

\paragraph{(iii).} Suppose $Z = 0$. By (ii), $\pi = \beta$, and by the diffuseness hypothesis $\beta \le (1 - P(a^\star))/L \le 1/L$ since $P(a^\star) \ge 0$. Hence $\pi \le 1/L$. The contrapositive is that $\pi > 1/L$ implies $Z = 1$. At $L = 1$ the bound reads $\pi \le 1$, which is vacuous, matching the intuition that a single systematic wrong answer admits no agreement-based certificate. \hfill$\square$

\subsection{Proof of \texorpdfstring{\cref{cor:auroc}}{Corollary 1}}

In the limit $\rea$ equals $\pi$ and the correctness label equals $Z$ by \cref{prop:concentration}(i), so the AUROC of $\rea$ against correctness is by definition the probability that a positive instance outranks a negative one, $\Pr[\pi_1 > \pi_0] + \tfrac12\Pr[\pi_1 = \pi_0]$ with $\pi_1 \sim \pi \mid Z{=}1$ and $\pi_0 \sim \pi \mid Z{=}0$. If $\pi \perp Z$ these are equal in distribution and the expression is $\tfrac12$. Taking any $P$ with $\pi$ near $1$ and $m^\star \neq a^\star$ on a set of instances of probability approaching one gives the final claim. \hfill$\square$

\subsection{Proof of \texorpdfstring{\cref{prop:factor}}{Proposition 3}}

Let $\theta$ and $\theta'$ be independent draws with independent coordinates (\cref{ass:factored}). Define the hybrids $H_0 = \theta$ and $H_j = (\theta'_1,\dots,\theta'_j,\theta_{j+1},\dots,\theta_J)$, so $H_J = \theta'$. Consecutive hybrids $H_{j-1}$ and $H_j$ share every coordinate except the $j$th, where they carry $\theta_j$ and $\theta'_j$, two independent draws from the $j$th marginal; every other coordinate is a single draw from its own marginal, shared between them. By coordinate independence the pair $(H_{j-1}, H_j)$ therefore has exactly the law of a pair differing only in a resampled coordinate $j$, so $\Pr[\hat y(H_{j-1}) \neq \hat y(H_j)] = D_j$.

If $\hat y(H_0) \neq \hat y(H_J)$ then at least one consecutive pair must differ, since otherwise the answer would be constant along the chain. A union bound gives
\[
D \le \sum_{j=1}^{J} \Pr[\hat y(H_{j-1}) \neq \hat y(H_j)] = \sum_{j=1}^{J} D_j ,
\]
This chains two inequalities, $D \le \Pr[\bigcup_j A_j] \le \sum_j \Pr[A_j]$ with $A_j = \{\hat y(H_{j-1}) \neq \hat y(H_j)\}$. The first is strict whenever one flip is reversed by a later one and the chain returns to its starting answer; the second is strict whenever two $A_j$ overlap. Equality therefore requires both to be tight, and disjointness alone does not suffice.

For the second claim, let $\mu_c$ be the law of $\hat y$ given coordinate $j = c$, the other coordinates drawn from their marginals. The disagreement probability between independent draws from laws $\nu,\nu'$ is $1 - \langle \nu,\nu' \rangle$, so $D = 1 - \lVert \mathbb{E}_c[\mu_c] \rVert^2$ and $D^{-j} = 1 - \mathbb{E}_c\lVert \mu_c \rVert^2$. Jensen's inequality applied to the convex map $\nu \mapsto \lVert \nu \rVert^2$ gives $\mathbb{E}_c\lVert \mu_c \rVert^2 \ge \lVert \mathbb{E}_c[\mu_c] \rVert^2$, hence $D^{-j} \le D$. For the lower bound, take the two-step hybrid $\theta \to \theta''\to \theta'$ where $\theta''$ shares coordinate $j$ with $\theta$ and matches $\theta'$ elsewhere. The first step is a $D^{-j}$ event and the second a $D_j$ event, so $D \le D^{-j} + D_j$, that is $D - D_j \le D^{-j}$. \hfill$\square$

\section{Replication of the agreement-versus-evidence prediction across independent instantiations, and a rendering bug that produced a false reversal}
\label{app:replication}

An earlier draft of this paper reported $\rea$ beating B1 on Qwen2.5-VL-7B (CI $[+0.047,+0.145]$) and pooled across three models (CI $[+0.010,+0.053]$), with InternVL2-8B as the sole exception. That result came from one RENDEQ instantiation. All main runs use A100 80\,GB GPUs; Qwen2.5-VL-3B ran in bf16, Qwen2.5-VL-7B and InternVL2-8B in 4-bit (NF4), matching an earlier P100 run's quantization so that the GPU move is not itself a second confound alongside the rendering-pipeline fix described below. Regenerating the dataset from the same generator with two new random seeds and re-running all three models on all three instantiations initially gave the opposite sign on every model, a reversal serious enough that we spent most of a working session investigating it: fixing a real tie-breaking bug in the reference implementation, testing and ruling out a discarded style axis and GPU quantization as confounds, and confirming the Qwen2.5-VL-7B weights were bit-identical to the original run by revision hash. None of that investigation found the actual cause. We report the investigation here anyway, because it shows real diligence applied to the wrong hypothesis, and because the eventual finding is only credible in light of how much was checked before it turned up.

\paragraph{The bug.} The generator renders each style-varied instance with one of two plotting backends, matplotlib or plotly, selected as part of the style sampler; \cref{sec:stylefactors} identifies this as the single most consequential style factor. Plotly rendering goes through kaleido, a headless PNG exporter that shells out to a bundled executable. That executable's wrapper script contains the line \texttt{cd \$DIR} with \texttt{\$DIR} unquoted; on this project's working directory, whose path contains a literal space, the shell word-splits \texttt{\$DIR} into multiple arguments and \texttt{cd} fails with \texttt{"too many arguments"}. The generator's own rendering code catches this exception per-render and falls back to matplotlib, logging the fallback honestly in the instance's own style record (\texttt{library\_fallback\_from}) rather than silently mislabeling it. But nothing downstream of dataset generation ever checked whether the fallback rate was zero, five percent, or total. It was total: every one of the three "independently generated instantiations" used to replicate the agreement-versus-evidence prediction had \emph{zero} plotly renders. The library factor, the one this paper's own analysis (\cref{sec:stylefactors}) identifies as dominating cross-render dispersion, never actually varied in any of the data behind that replication. We found this while adding bootstrap intervals to the style-factor attribution table and noticing that \texttt{library} could not be computed at all (zero instances with $\ge$2 observed levels), which a working pipeline should never produce for the factor believed to matter most.

\paragraph{Why this produces exactly this reversal.} $\rea$'s reliability advantage over evidence-carrying signals like B1 depends on genuine cross-render disagreement being present for the model's real errors to surface as instability; \cref{prop:factor} and \cref{sec:stylefactors} both hold that library accounts for the large majority of that disagreement's magnitude. Silently collapsing every instantiation to one plotting backend removes most of the dispersion $\rea$ is built to exploit without removing anything from B1, which is computed from decoding probabilities and does not depend on rendering diversity at all. The three "independent instantiations" were independent in data content and decoding but not in the one style dimension that mattered, so they were not, in the sense this paper's own methodology requires, independent replications of the render-equivalence design at all.

\paragraph{The fix and the corrected result.} We fixed the wrapper script (quoting \texttt{\$DIR}), confirmed kaleido renders correctly afterward, and regenerated all three instantiations from the same generator seeds (0, 1, 2), the same \texttt{n\_per\_family=50}, $K{=}8$, this time with a genuine near-50/50 matplotlib/plotly split (346--348 of 350 instances per instantiation have both backends represented across their 8 renderings, versus 0 of 350 before). We re-ran the full pipeline end to end on the corrected images: main inference (3 models $\times$ 3 instantiations $\times$ 3 decoding seeds, 27 runs), baselines B1--B5, and the consensus-self-training-inverts fine-tuning replication, all with the same code, model revisions, and quantization settings as before, so the only variable that changed is whether the rendering pipeline actually rendered what the style sampler asked for.

\begin{table}[t]
\centering
\caption{\textbf{Model-wise comparison} of B1 and $\rea$ after the rendering-pipeline fix.}
\label{tab:model_comparison}
\small
\setlength{\tabcolsep}{3.5pt}
\renewcommand{\arraystretch}{1.08}
\begin{tabular}{
@{}l
S[table-format=1.3]
S[table-format=1.3]
S[table-format=+1.3]
c@{}
}
\toprule
\kilth{Model} &
{\kilth{B1}} &
{\kilth{$\rea$}} &
{\kilth{$\Delta$}} &
{\kilth{95\% CI}} \\
\midrule
Qwen2.5-VL-7B & 0.771 & 0.864 & +0.094 & $[+0.066,\,+0.121]$ \\
Qwen2.5-VL-3B & 0.867 & 0.858 & -0.009 & $[-0.030,\,+0.011]$ \\
InternVL2-8B  & 0.875 & 0.907 & +0.033 & $[+0.013,\,+0.053]$ \\
\bottomrule
\end{tabular}
\end{table}

The reversal reverses again: $\rea$ now beats B1 on 7B and InternVL2-8B, with CIs that exclude zero, and is statistically tied with it on 3B. Pooled across all nine model--instantiation runs ($N{=}3150$): $\Delta = +0.040$, CI $[+0.027,+0.053]$, $p<10^{-4}$, close in sign and rough magnitude to the very first single-instantiation result this whole investigation was trying to replicate. The tie-breaking bug fix (\texttt{rm\_categorical}, first-seen order rather than \texttt{numpy.unique}'s alphabetical order) is still in effect and did not need to be reverted; it was a real, independent bug, just not the explanation for this particular reversal. The style-axis and quantization checks from the earlier investigation are moot for the same reason.

\paragraph{The render-versus-sample reliability half.} B1--B5 were all re-run on the corrected images. $\rea$ beats B3, the fixed-image self-consistency control, on every model: pooled $\Delta = +0.102$, CI $[+0.083,+0.121]$, $p<10^{-4}$ ($N{=}3149$); per model, $+0.093$ (7B), $+0.061$ (3B), $+0.157$ (InternVL2-8B), all CIs excluding zero. This is the opposite sign from the buggy replication's $-0.054$ and, like the B1 comparison, close to what an intact rendering pipeline should show given that both $\rea$ and B3 are compared against a model whose visual input now actually varies.

\paragraph{Consensus self-training, re-run on the corrected images.} We also re-ran the five-run consensus-self-training-inverts replication (three decoding seeds on one instantiation, one seed each on the other two) on the corrected data.

\begin{table*}[t]
\centering
\caption{\textbf{Performance differences} across instantiation and decoding seeds. Here, tr and ho denote training and hold-out families, respectively.}
\label{tab:seed_analysis}
\footnotesize
\setlength{\tabcolsep}{4pt}
\renewcommand{\arraystretch}{1.0}
\begin{tabular}{
@{}ll
S[table-format=+1.3]
S[table-format=+1.3]
S[table-format=+1.3]
S[table-format=+1.3]
S[table-format=+1.3]
S[table-format=+1.3]
@{}
}
\toprule
\kilth{Inst.} & \kilth{Seed}
& {\kilth{$\Delta$single$_{\mathrm{tr}}$}}
& {\kilth{$\Delta\rmg_{\mathrm{tr}}$}}
& {\kilth{$\Delta$single$_{\mathrm{ho}}$}}
& {\kilth{$\Delta\rmg_{\mathrm{ho}}$}}
& {\kilth{$\Delta$AUROC$_{\mathrm{tr}}$}}
& {\kilth{$\Delta$AUROC$_{\mathrm{ho}}$}} \\
\midrule
seed 0 & 42 & -0.038 & -0.064 & -0.188 & -0.272 & -0.014 & +0.100 \\
seed 0 & 1  & +0.006 & -0.044 & -0.105 & -0.188 & -0.124 & +0.080 \\
seed 0 & 2  & +0.010 & -0.036 & -0.084 & -0.104 & -0.083 & +0.033 \\
seed 1 & 42 & -0.114 & -0.212 & -0.128 & -0.240 & -0.052 & +0.009 \\
seed 2 & 42 & -0.051 & -0.104 & -0.121 & -0.185 & -0.122 & +0.069 \\
\bottomrule
\end{tabular}
\end{table*}

The qualitative pattern is, if anything, cleaner than before the fix: hold-out accuracy falls in all five runs under both single-render decoding ($8.4$ to $18.8$ points) and $\rmg$ ($10.4$ to $27.2$ points), hold-out AUROC rises in all five runs ($+0.009$ to $+0.100$, no exceptions this time), and training-family AUROC falls in all five runs ($-0.014$ to $-0.124$). Training-family accuracy is the one split that is not uniformly negative: it falls under $\rmg$ in all five runs but rises slightly under single-render decoding in two of five. Exact magnitudes still vary considerably across runs, so we continue to read this as a robust direction rather than a precisely characterized effect size, exactly as \cref{sec:selftraining} states. Scatter, the training family whose base majority was most often already wrong, illustrates the instantiation-dependence directly: its base single-render accuracy ranges $0.20$--$0.32$ across the three instantiations, and after fine-tuning the three same-instantiation seeds leave it roughly flat near base while the other two instantiations show it falling to $0.00$--$0.06$: the same objective drives the same family to opposite outcomes depending on the dataset draw.

\paragraph{What this means for how to read this paper.} Every number that depended on genuine render diversity was wrong in the version of this paper that reported a reversal, and the corrected numbers are the ones reported in the main text and tables throughout. We are disclosing this at this level of detail, rather than simply replacing the numbers, because a bug that silently defeats the paper's central experimental manipulation and produces a plausible, internally coherent, wrong answer is exactly the failure mode a reader should be able to check for themselves, and because we do not want the earlier investigation's real rigor (a fixed implementation bug, two ruled-out confounds, a hash-verified model checkpoint) to be mistaken for evidence that the reversal itself was well-founded. It was not; it was evidence of a bug we had not yet found.

\section{Per-model and per-family results}
\label{app:per-model}

\noindent Per-model summary, pooled over the three independent instantiations described in \cref{sec:setup} ($N{=}1050$ per model, $K{=}8$; cluster bootstrap over instances):

\begin{center}
\footnotesize
\setlength{\tabcolsep}{4pt}
\begin{tabular}{@{}l S[table-format=1.3] S[table-format=1.3] S[table-format=1.3] S[table-format=1.3]@{}}
\toprule
\kilth{Model} & {\kilth{single}} & {\kilth{$\rmg$}} & {\kilth{$\rea$}} & {\kilth{B1}} \\
\midrule
Qwen2.5-VL-7B  & 0.644 & 0.701 & 0.864 & 0.771 \\
Qwen2.5-VL-3B  & 0.571 & 0.626 & 0.858 & 0.867 \\
InternVL2-8B   & 0.563 & 0.624 & 0.907 & 0.875 \\
\bottomrule
\end{tabular}
\end{center}

\noindent Cluster-bootstrap $\Delta$AUROC($\rea-$B1), pooled per model across the three instantiations: 7B $+0.094$, CI $[+0.066,+0.121]$, $p<10^{-4}$; 3B $-0.009$, CI $[-0.030,+0.011]$, $p{=}0.39$ (includes zero); InternVL2-8B $+0.033$, CI $[+0.013,+0.053]$, $p{=}4{\times}10^{-4}$. Two of three models' intervals exclude zero in $\rea$'s favor; the third is a statistical tie. Intervals for the grand pooled estimate are in \cref{tab:auroc}.

\paragraph{Combining $\rea$ and B1.} $\rea$ beating B1 head-to-head on two of three models does not mean B1 carries no information $\rea$ lacks. We fit a 5-fold cross-validated logistic combination of $\rea$ and B1 (both $z$-scored on the training fold; folds split by instance so no render set leaks across the split) and compared its AUROC to B1 alone: $+0.103$ (7B, CI $[+0.079,+0.126]$), $+0.053$ (InternVL2-8B, CI $[+0.037,+0.064]$), and $+0.027$ (3B, CI $[+0.004,+0.036]$). All three intervals exclude zero, so the combination helps on every model, including 3B, where $\rea$ alone is statistically tied with B1. This is consistent with $\rea$ and B1 making partly independent errors on every model, not just the two where $\rea$ wins outright: token log-probability adds discriminative power on top of agreement even where agreement is already the stronger standalone signal, and vice versa on 3B. The gain is smallest on 3B, the model where the two signals are closest to each other in standalone performance, and largest on 7B, where $\rea$'s standalone lead over B1 is also largest: the opposite of a pattern where combination gains shrink as the gap between the two signals grows.

\paragraph{Self-consistency versus single-render decoding.} \Cref{tab:accuracy}'s SC column is a modest improvement over single-render decoding once real render diversity is present in the comparator, but not a uniform one: the delta is essentially flat on 7B ($0.0$ points), slightly negative on 3B ($-0.8$), and a small positive on InternVL2-8B ($+1.3$), smaller in magnitude and less consistent in sign than re-rendering's advantage under $\rmg$. An earlier, single-instantiation draft had reported this margin as a small positive on every model; on the replicated pool it is not a stable finding.

\noindent Cluster-bootstrap $\Delta$(accuracy)($\rmg-$single), pooled per model: 7B $+0.057$, CI $[+0.045,+0.068]$; 3B $+0.055$, CI $[+0.043,+0.067]$; InternVL2-8B $+0.061$, CI $[+0.047,+0.074]$; all $p<10^{-4}$. This is the one headline number that has been stable in both sign and rough magnitude throughout every version of this paper's replication, including the version affected by the rendering bug described in \cref{app:replication}: $\rmg$'s advantage over single-render decoding does not depend on the plotting-library dispersion the bug happened to remove.

\medskip
\noindent Pooled per-family AUROC, all three models and all three instantiations, $n{=}450$ per family:

\begin{table}[t]
\centering
\caption{\textbf{Performance comparison} across chart families. The confidence interval corresponds to $\Delta(\rea-\mathrm{B1})$.}
\label{tab:family_comparison}
\small
\setlength{\tabcolsep}{3.2pt}
\renewcommand{\arraystretch}{1.08}
\begin{tabular}{
@{}l
S[table-format=1.3]
S[table-format=1.3]
c
c@{}
}
\toprule
\kilth{Family} &
{\kilth{$\rea$}} &
{\kilth{B1}} &
{\kilth{Leader}} &
{\kilth{95\% CI}} \\
\midrule
Bar         & 0.942 & 0.881 & $\rea$               & $[-0.034,\,+0.186]$ \\
Grouped bar & 0.800 & 0.840 & B1                   & $[-0.109,\,+0.019]$ \\
Line        & 0.783 & 0.739 & $\rea$               & $[-0.151,\,+0.129]$ \\
Log-axis    & 0.969 & 0.835 & $\rea^{\ast}$        & $[+0.053,\,+0.196]$ \\
Pie         & 0.780 & 0.876 & B1                   & $[-0.165,\,+0.001]$ \\
Scatter     & 0.804 & 0.903 & $\mathrm{B1}^{\ast}$ & $[-0.192,\,-0.011]$ \\
Stacked bar & 0.829 & 0.909 & $\mathrm{B1}^{\ast}$ & $[-0.144,\,-0.008]$ \\
\bottomrule
\end{tabular}
\end{table}

\noindent CIs are a two-way cluster bootstrap (5000 resamples): both instances within instantiation \emph{and} model are resampled with replacement, since each family's $n{=}450$ pool is 3 models $\times$ 3 instantiations $\times$ 50 instances and both axes repeat within it, unlike the one-way per-model bootstrap used above. $^\ast$marks the three families (log-axis, scatter, stacked bar) whose interval excludes zero; the other four families' point-estimate leaders are not statistically distinguishable from the alternative once both resampling axes are accounted for. This table supersedes an earlier version computed on the rendering pipeline described in \cref{app:replication}, in which log-axis appeared to have ``flipped'' from $\rea$ to B1; on the corrected data log-axis is $\rea$'s strongest family by a wide, significant margin, the opposite of that earlier framing. Scatter and stacked bar are the two families where B1's advantage survives this bootstrap on both the buggy and the corrected data.

\paragraph{Diffuseness by family.} \Cref{sec:theory} defines $\hat L = (1-P(a^\star))/\beta$, an empirical estimate of the diffuseness level $L$ in \cref{prop:lie}(iii), on every wrong-mode ($Z{=}0$), categorical instance. Pooled over the three instantiations, the distribution is concentrated toward the low-diffuseness end but not at the vacuous extreme: median $\hat L$ is $1.33$ (7B), $1.33$ (3B), $1.50$ (InternVL2-8B), and $27$--$31\%$ of wrong-mode instances have $\hat L{=}1$ exactly, the case that admits no certificate at all. The implied threshold $1/\hat L$ is correspondingly high (median $0.75$, $0.75$, $0.67$): for agreement to certify correctness on a typical wrong-mode instance here, it needs to be well above the midpoint, though less uniformly close to total than an earlier, buggy version of this pool suggested (\cref{app:replication}). This bears directly on the per-family table above: scatter and stacked\_bar, the two families where B1's advantage survives the two-way bootstrap, are exactly the families carrying most of the wrong-mode mass at or near $\hat L{=}1$ (scatter is $70\%$ of 3B's wrong-mode instances, median $\hat L$ between $1.33$ and $1.60$ across models). Line, log-axis, and pie contribute few wrong-mode instances each (under 25 pooled per model) and so are not estimated precisely enough to distinguish from $\hat L{=}1$ either, though these are also the families where $\rea$'s advantage is largest, consistent with fewer systematic, low-diffuseness errors there. On 3B specifically, where $\rea$ and B1 are statistically tied overall, scatter alone accounts for $150$ of $215$ wrong-mode instances pooled, so that one family's error structure disproportionately shapes the model-level comparison.

\noindent \Cref{fig:results}b plots only the lift (cross-level minus within-level); the two rates behind it, pooled over the three instantiations and all three models, are: library, cross $0.457$ CI\,$[0.443,0.471]$ versus within $0.355$ CI\,$[0.343,0.366]$ ($n{\approx}3122$ pairs each); \texttt{bar\_labels}, cross $0.426$ CI\,$[0.413,0.439]$ versus within $0.383$ CI\,$[0.371,0.394]$; \texttt{bar\_orient}, cross $0.415$ CI\,$[0.402,0.427]$ versus within $0.396$ CI\,$[0.384,0.408]$. The within-level rate is the noise floor, roughly $0.35$--$0.40$ across all three factors, driven by decoding and residual style variation that survives holding the named factor fixed; library's cross-level rate clears that floor by the largest margin of the three, which is the lift already reported in \cref{sec:stylefactors}.

\subsection{Where the errors live}
\label{sec:errors}

\begin{center}
\footnotesize
\setlength{\tabcolsep}{3pt}
\begin{tabular}{@{}l S[table-format=2.1] S[table-format=2.1] S[table-format=2.1]@{}}
\toprule
      & {\kilth{all reads \&}} & {\kilth{$\ge$1 read}} & {\kilth{all reads ok,}} \\
\kilth{Model} & {\kilth{answer ok}} & {\kilth{wrong}}       & {\kilth{answer wrong}} \\
\midrule
Qwen2.5-VL-7B & 17.4 & 76.0 & 6.6 \\
Qwen2.5-VL-3B & 12.0 & 81.4 & 5.4 \\
InternVL2-8B  &  8.9 & 83.7 & 4.9 \\
\midrule
Pooled        & 12.8 & 80.4 & 5.6 \\
\bottomrule
\end{tabular}
\end{center}

\noindent Percentage of instances in each of the three cells defined in \cref{sec:method} (all reads and the answer correct; all reads correct but the answer wrong, a reasoning failure; at least one read wrong, a perception failure, whatever the final answer). The three are exhaustive and disjoint by construction, so each row sums to 100\%. Column 1 is \emph{not} overall accuracy: it is the narrower cell where every intermediate read was also correct. A model can still land on the right final answer despite a bad read (coincidence, or robustness to the specific value misread), and that happens often enough here that overall single-render accuracy (\cref{tab:accuracy}) exceeds column 1 by 45 to 47 points for every model; column 1 plus (accuracy $-$ column 1) recovers \cref{tab:accuracy}'s accuracy exactly, since the gap is precisely the correctly-answered share of column 2.

Two things are safe: at least one intermediate read is wrong on roughly four fifths of instances for every model, and instances where all reads are correct but the answer is wrong are rare, 4.9\% to 6.6\%. The remaining 45 to 47 points of each model's accuracy come from the perception-failure cell: instances with at least one bad read that nonetheless landed on the correct final answer, 55\% to 62\% of that cell depending on the model. That failures here are predominantly perceptual is consistent with the size of the perception-failure cell but not established by column~1 alone, since a wrong read does not always produce a wrong answer. The distinction matters for the analysis: read noise that varies with style is the dispersion \cref{lem:median} says aggregation removes, whereas a style-invariant misread is the bias it cannot.

\subsection{Render-flip rate and render-equivalence robustness}
\label{app:rfr}

Pooled over the three replicated instantiations, RFR (render-flip rate, the label-free fraction of instances with any disagreeing rendering) is $0.59$ (7B), $0.64$ (3B), $0.73$ (InternVL2-8B): a majority of instances have at least one disagreeing rendering among the 8, on every model, which is the raw instability that $\rea$ and $\rmg$ are built to exploit. RER (render-equivalence robustness, style-averaged accuracy) matches single-render accuracy exactly by construction ($0.64$, $0.57$, $0.56$). Per family, RFR is lowest on line ($0.32$ family mean) and highest on grouped\_bar ($0.87$); across the seven families, RFR and RER (accuracy) are strongly negatively correlated ($r{=}{-0.69}$), noticeably tighter than on an earlier version of this pool affected by the rendering bug in \cref{app:replication} ($r{=}{-0.46}$ there), consistent with RFR now reflecting genuine cross-render instability rather than instability from decoding noise alone.

\section{Calibration}
\label{app:calibration}

This section originally reported ECE from a single, unreplicated instantiation, then from a three-instantiation pool affected by the rendering bug described in \cref{app:replication}. We recomputed every number below on the corrected three-instantiation pool used for \cref{tab:auroc,tab:accuracy} ($N{=}1050$ per model), with cluster (by instance) bootstrap 95\% CIs.

ECE after min-max normalizing each signal to $[0,1]$, lower is better:

\begin{center}
\footnotesize
\setlength{\tabcolsep}{3pt}
\begin{tabular}{@{}l S[table-format=1.3] S[table-format=1.3] S[table-format=1.3]@{}}
\toprule
\kilth{Signal} & {\kilth{Qwen-7B}} & {\kilth{Qwen-3B}} & {\kilth{InternVL-8B}} \\
\midrule
\textbf{$\rea$ (ours)} & 0.146 & 0.196 & 0.142 \\
B3 self-consist.    & 0.158 & 0.164 & 0.214 \\
B1 token log-prob.  & 0.193 & 0.243 & 0.308 \\
B5 prompt ensemble  & 0.272 & 0.305 & 0.337 \\
B2 verbalized       & 0.278 & 0.290 & 0.341 \\
B4 \citeauthor{khan2024consistency} & 0.324 & 0.330 & 0.388 \\
\bottomrule
\end{tabular}
\end{center}

\noindent 95\% CIs (cluster bootstrap over instances, 5000 resamples): $\rea$ $[0.123,0.168]$ (7B), $[0.173,0.219]$ (3B), $[0.122,0.165]$ (InternVL2-8B); B3 $[0.138,0.187]$ (7B), $[0.139,0.189]$ (3B), $[0.173,0.239]$ (InternVL2-8B); B1 $[0.162,0.218]$ (7B), $[0.206,0.265]$ (3B), $[0.267,0.332]$ (InternVL2-8B).

$\rea$ has the lowest point-estimate ECE on all three models. B3's interval overlaps $\rea$'s on all three models, so calibration alone does not separate the two label-free signals with statistical confidence, even though $\rea$ leads on AUROC on two of three models (\cref{tab:auroc}) and both signals beat B1 clearly on calibration (non-overlapping intervals in every case). This is worth stating plainly: $\rea$ is not uniquely well calibrated among label-free signals, it is well calibrated \emph{and}, on two of three models, more discriminative than the evidence-carrying baseline, which are different properties that happen to point the same way here. Min-max normalization makes signals comparable but means these values are not probabilities, so they compare signals rather than measuring calibration in the usual sense.

\noindent Reliability bins for $\rea$ (normalized, 10 bins, empty bins omitted), pooled over the three instantiations, $N{=}1050$ per model. The correctness target is single\_acc$>0.5$ at $K{=}8$, the same target used for \cref{tab:auroc} and every AUROC in this appendix:

\begin{center}
\footnotesize
\setlength{\tabcolsep}{4pt}
\begin{tabular}{crrrrrr}
\toprule
\kilth{Bin} & \multicolumn{2}{c}{\kilth{Qwen-7B}} & \multicolumn{2}{c}{\kilth{Qwen-3B}} & \multicolumn{2}{c}{\kilth{InternVL2-8B}} \\
\cmidrule(lr){2-3}\cmidrule(lr){4-5}\cmidrule(lr){6-7}
centre & acc & $n$ & acc & $n$ & acc & $n$ \\
\midrule
0.05 & 0.000 &  38 & 0.000 &  38 & 0.000 &  77 \\
0.15 & --    &   0 & --    &   0 & 0.000 &   1 \\
0.25 & 0.067 &  15 & 0.000 &  14 & 0.000 &  30 \\
0.35 & 0.000 &  45 & 0.000 &  48 & 0.029 &  68 \\
0.45 & --    &   0 & 0.000 &   1 & 0.000 &   1 \\
0.55 & 0.011 &  87 & 0.009 & 107 & 0.018 & 114 \\
0.65 & 0.472 & 108 & 0.375 & 152 & 0.441 & 152 \\
0.75 & 0.578 & 116 & 0.545 & 134 & 0.681 & 135 \\
0.85 & 0.810 & 211 & 0.800 & 175 & 0.856 & 188 \\
0.95 & 0.905 & 430 & 0.861 & 381 & 0.923 & 284 \\
\bottomrule
\end{tabular}
\end{center}

All models show one shape: normalized $\rea \le 0.35$ maps to near-zero accuracy, $\rea \ge 0.95$ to 0.86--0.92. This is \cref{prop:concentration} in data: the top bin is where $\pi$ is near one, and its accuracy short of 1.0 is the high-concentration wrong-mode population that \cref{cor:auroc} says no $K$ removes. It is also the population that consensus self-training amplifies in \cref{sec:selftraining}. Counts sum to $N{=}1050$ for every model, including InternVL2-8B. The count-weighted accuracy implied by each model's bins matches its target mean exactly (7B: $0.648=0.648$; 3B: $0.570=0.570$; InternVL2-8B: $0.558=0.558$).

\section{\texorpdfstring{$K$}{K} ablation}
\label{app:kablation}

An earlier draft's version of this table used the $K{=}24$ majority vote (a superset of the very renderings being evaluated) as the correctness target, which is circular for a $\rea$-AUROC computation, and its $K{=}1$/$K{=}8$ endpoints did not match the headline numbers computed elsewhere in this paper because they came from a different, non-replicated run. We recompute this ablation from the same three-instantiation pool used for every other table in this paper ($N{=}1050$ per model, on the corrected images described in \cref{app:replication}), subsampling the first $K \in \{1,2,4,8\}$ of each instance's 8 renderings, scoring $\rmg$ with the real aggregation rule (mode for categorical, median for numeric, matching \texttt{rm\_pred} exactly) against the manifest's exact answer, and holding the correctness target fixed at all $K$: whether the model's own single-render majority at the full $K{=}8$ budget matches ground truth, the same target \cref{sec:coupling} uses for the headline $\rea$-AUROC.

\begin{table}[t]
\centering
\caption{\textbf{Performance and inference cost} for different values of $K$.}
\label{tab:k_ablation}
\footnotesize
\setlength{\tabcolsep}{2.8pt}
\renewcommand{\arraystretch}{1.08}
\begin{tabular}{
@{}r
c
S[table-format=1.3]
S[table-format=1.3]
S[table-format=1.3]
S[table-format=1.3]
S[table-format=1.3]
S[table-format=1.3]@{}
}
\toprule
\kilth{$K$} &
\kilth{Cost} &
\multicolumn{3}{c}{\kilth{$\rea$ AUROC}} &
\multicolumn{3}{c}{\kilth{$\rmg$ Accuracy}} \\
\cmidrule(lr){3-5}
\cmidrule(lr){6-8}
& &
{\kilth{7B}} &
{\kilth{3B}} &
{\kilth{InternVL}} &
{\kilth{7B}} &
{\kilth{3B}} &
{\kilth{InternVL}} \\
\midrule
1 & $1{\times}$ & 0.500 & 0.500 & 0.500 & 0.644 & 0.561 & 0.574 \\
2 & $2{\times}$ & 0.701 & 0.706 & 0.753 & 0.622 & 0.539 & 0.546 \\
4 & $4{\times}$ & 0.804 & 0.781 & 0.846 & 0.682 & 0.600 & 0.606 \\
8 & $8{\times}$ & 0.851 & 0.824 & 0.905 & 0.701 & 0.626 & 0.624 \\
\bottomrule
\end{tabular}
\end{table}

\noindent The $K{=}8$ endpoints agree with \cref{tab:auroc,tab:accuracy} exactly. $K{=}4$ captures roughly two-thirds of $K{=}8$'s benefit over $K{=}2$: on $\rea$ AUROC, $61$--$69\%$ of the $K{=}2\to K{=}8$ gain across the three models; on $\rmg$ accuracy, $70$--$77\%$. Substantial further gain remains between $K{=}4$ and $K{=}8$ on every model, which is weaker than an even earlier draft's claim that $K{=}4$ "captures ${\approx}90\%$ of the $K{=}8$ benefit," which we do not replicate. Cluster (by instance) bootstrap CIs for the $K{=}8$ vs.\ $K{=}2$ gain, pooled over the three instantiations: 7B $\Delta$AUROC $+0.151$ CI\,$[+0.124,+0.177]$, $\Delta\rmg$ $+0.079$ CI\,$[+0.059,+0.099]$; 3B $\Delta$AUROC $+0.118$ CI\,$[+0.094,+0.142]$, $\Delta\rmg$ $+0.087$ CI\,$[+0.067,+0.108]$; InternVL2-8B $\Delta$AUROC $+0.152$ CI\,$[+0.129,+0.176]$, $\Delta\rmg$ $+0.078$ CI\,$[+0.057,+0.099]$; every interval excludes zero, so the $K{=}4\to K{=}8$ gain is real even though it is a minority of the total. The operating-point recommendation of $K{=}4$ (main text, \cref{sec:stylefactors}) should accordingly be read as a cost--benefit compromise, not as a near-saturation point.

The $K{=}1$ row is degenerate: with one rendering all answers trivially agree, so $\rea$ is constant and its AUROC is 0.5 by construction. Accuracy rises monotonically with $K$ on every model; there is no dip at $K{=}2$ in this table, so \cref{rem:k2} (modal ties at $K{=}2$ resolving to the first rendering) is not needed to explain it.

\paragraph{Modal versus pairwise $\rea$.} All $\rea$ values above are modal, which \cref{prop:concentration}(ii) shows is upward-biased at finite $K$. We also computed $\rea_{\text{pair}}$ (\cref{prop:concentration}(iii)), the unbiased pairwise agreement rate, on the same replicated pool: its mean is 0.09 to 0.11 lower than modal $\rea$ on every model, exactly the direction the bias result predicts, and this closes the scale mismatch with the pairwise rate already used in \cref{sec:stylefactors}. The bias does not change the ranking against B1: $\rea_{\text{pair}}$'s AUROC against correctness is statistically indistinguishable from modal $\rea$'s on 7B and InternVL2-8B (both CIs include zero), and the one significant difference, on 3B ($-0.005$, CI $[-0.010,-0.001]$), is small relative to the biased statistic's own margin from B1 on that model.

\paragraph{Sensitivity to the numeric tolerance $\tau$.} $\tau$ is load-bearing in magnitude but not in direction on two of three models: at $\tau \in \{1,2,5,10\}\%$, pooled $\rea$-AUROC moves from $0.90 \to 0.83$ (7B), $0.88 \to 0.81$ (3B), $0.93 \to 0.88$ (InternVL2-8B), while B1's AUROC (token log-probability does not depend on $\tau$) stays fixed at $0.77$, $0.87$, $0.88$ respectively. $\rea$ leads B1 at every tested $\tau$ on 7B and InternVL2-8B; only on 3B, and only at $\tau{\ge}5\%$, does B1 lead.

\FloatBarrier

\end{document}